\documentclass[final,12pt]{colt2021}

\usepackage{xcolor}

\definecolor{brickred} {rgb}{0.6,0,0}
\definecolor{jaehoblue}{rgb}{0,0,0.8}

\usepackage{mathtools}
\usepackage{amsmath,amssymb}
\usepackage{graphicx}

\makeatletter
\def\set@curr@file#1{\def\@curr@file{#1}} 
\makeatother
\usepackage[load-configurations=version-1]{siunitx} 

\newtheorem{innercustomthm}{Theorem}
\newenvironment{customthm}[1]
  {\renewcommand\theinnercustomthm{#1}\innercustomthm}
  {\endinnercustomthm}

\newtheorem{claim}[theorem]{Claim}

\newtheorem{mydefinition}{Definition}
\newtheorem*{theorem*}{Theorem}

\definecolor{Gray}{gray}{0.9}

\newcommand{\relu}{\textsc{ReLU}}
\newcommand{\id}{\textsc{Id}}
\newcommand{\step}{\textsc{Step}}

\title{Provable Memorization via Deep Neural Networks using Sub-linear Parameters}
\usepackage{times}
\coltauthor{%
 \Name{Sejun Park} \Email{sejun.park@vectorinstitute.ai}\\
 \addr Vector Institute for Artificial Intelligence
 \AND
 \Name{Jaeho Lee} \Email{jaeho-lee@kaist.ac.kr}\\
 \addr Korea Advanced Institute of Science and Technology
 \AND
 \Name{Chulhee Yun} \Email{chulheey@mit.edu}\\
 \addr Massachusetts Institute of Technology
 \AND
 \Name{Jinwoo Shin} \Email{jinwoos@kaist.ac.kr}\\
 \addr Korea Advanced Institute of Science and Technology%
}
\begin{document}

\maketitle
\begin{abstract}
It is known that $O(N)$ parameters are sufficient for neural networks to memorize arbitrary $N$ input-label pairs. By exploiting depth, we show that $O(N^{2/3})$ parameters suffice to memorize $N$ pairs, under a mild condition on the separation of input points. In particular, deeper networks (even with width $3$) are shown to memorize more pairs than shallow networks, which also agrees with the recent line of works on the benefits of depth for function approximation. We also provide empirical results that support our theoretical findings.
\end{abstract}

\section{Introduction}
The modern trend of over-parameterizing neural networks has shifted the focus of deep learning theory from analyzing their expressive power toward understanding the generalization capabilities of neural networks. While the celebrated universal approximation theorems state that over-parameterization enables us to approximate the target function with a smaller error \citep{cybenko89,pinkus99}, the theoretical gain is too small to satisfactorily explain the observed benefits of over-parameterizing already-big networks. Instead of ``how well can models fit,'' the question of ``why models do not overfit'' has become the central issue \citep{zhang17}.

Ironically, a recent breakthrough on the phenomenon known as the \textit{double descent} \citep{belkin19,nakkiran20} suggests that answering the question of ``how well can models fit'' is in fact an essential element in fully characterizing their generalization capabilities.
In particular, the double descent phenomenon characterizes two different phases according to the capability/incapability of the network size for memorizing training samples.
If the network size is insufficient for memorization, the traditional bias-variance trade-off occurs. However, after the network reaches the capacity that memorizes the dataset, i.e., ``interpolation threshold,'' larger networks exhibit better generalization.
Under this new paradigm, identifying the \textit{minimum size of networks} for memorizing \textit{finite} input-label pairs becomes a key issue, rather than function approximation that considers \textit{infinite} inputs.

The memory capacity of neural networks is relatively old literature, where researchers have studied the minimum number of parameters for memorizing arbitrary $N$ input-label pairs. 
Existing results showed that $O(N)$ parameters (i.e., weights and biases) are sufficient for various activation functions \citep{baum88,huang98,huang03,yun19,vershynin20}.
On the other hand, \cite{sontag97} established the negative result that for any network using analytic definable activation functions with $o(N)$ parameters, there exists a set of $N$ input-label pairs that the network cannot memorize.
The sub-linear number of parameters also appear in a related topic, namely the VC-dimension of neural networks.
It has been proved that there exists a set of $N$ inputs such that a neural network with $o(N)$ parameters can ``shatter,'' i.e., memorize arbitrary labels \citep{maass97,bartlett19}.
Comparing the two results on $o(N)$ parameters, \cite{sontag97} showed that \emph{not all} sets of $N$ inputs can be memorized for arbitrary labels, whereas \cite{bartlett19} showed that \emph{at least one} set of $N$ inputs can be shattered.
This suggests that there may be a reasonably large family of $N$ input-label pairs that can be memorized with $o(N)$ parameters, which is our main interest.

\subsection{Summary of results}\label{sec:contribution}
In this paper, we identify a mild condition satisfied by many practical datasets, and show that $o(N)$ parameters suffice for memorizing such datasets.
To bypass the negative result by \cite{sontag97}, we introduce a condition to the set of inputs, called the $\Delta$-separateness.
\begin{mydefinition}\label{def:separated}
For $\mathcal X\subset\mathbb{R}^{d_x}$, we say $\mathcal X$ is $\Delta$-separated if $$\sup_{x,x^\prime\in\mathcal X:x\ne x^\prime}\|x-x^\prime\|_2<\Delta\times\inf_{x,x^\prime\in\mathcal X:x\ne x^\prime}\|x-x^\prime\|_2.$$
\end{mydefinition}
This condition requires that the ratio of the maximum distance to the minimum distance between distinct points is bounded by $\Delta$.
Here, the condition is milder when $\Delta$ is bigger.
Notice that any given finite set of (distinct) inputs is $\Delta$-separated for some $\Delta$, so one might ask why $\Delta$-separateness is different from having distinct inputs in a dataset.
The key difference is that even if the number of data points $N$ grows, the ratio of the maximum to the minimum should remain bounded by some $\Delta$.
Given the discrete nature of computers, many practical datasets satisfy $\Delta$-separateness, as we will see shortly.
Also, this condition is more general than the minimum distance assumptions ($\forall i,  \|x_i\|_2 = 1$, $\forall i \neq j, \|x_i - x_j\|_2 \geq \rho > 0$) that are employed in existing theoretical results \citep{hardt17,vershynin20}. To see this, note that the minimum distance assumption implies $2/\rho$-separateness.
In our theorem statements, we will use the phrase ``$\Delta$-separated set of $N$ pairs'' to refer to $N$ input-label pairs, where the set of inputs is $\Delta$-separated.

In our main theorem sketched below, we prove the sufficiency of $o(N)$ parameters for memorizing {any} $\Delta$-separated set of $N$ pairs (i.e., any $\Delta$-separated set of $N$ inputs {with arbitrary labels}) even for large $\Delta$. More concretely, our result is of the following form:
\begin{customthm}{1}[Informal]\label{thm:lw-informal}
For any $w\in[2/3,1]$, there exists a $O(N^{2-2w}+\log\Delta)$-layer, $O(N^w+\log\Delta)$-parameter fully-connected network using a sigmoidal or $\relu$ activation function that can memorize any $\Delta$-separated set of $N$ 
pairs.
\end{customthm}
Theorem \ref{thm:lw-informal} implies that for $w=2/3$ and $\Delta=2^{O(N^{2/3})}$, $O(N^{2/3})$ parameters are sufficient for memorizing any $\Delta$-separated set of $N$ pairs.
Here, we can check from Definition~\ref{def:separated} that the $\log \Delta$ term does not usually dominate the depth or the number of parameters, especially for modern deep architectures and practical datasets. For example, it is easy to check that any dataset consisting of $3$-channel images (values from $\{0,1,\dots,255\}$) of size $a\times b$ satisfies $\log\Delta<(9+\frac12\log(ab))$ (e.g., $\log\Delta<17$ for the ImageNet dataset), which is often much smaller than the depth of modern deep architectures. 

For practical datasets, we can show that networks with parameters fewer than the number of pairs can successfully memorize the dataset.
For example, in order to perfectly classify one million images in the ImageNet dataset\footnote{http://www.image-net.org/} with $1000$ classes, our result shows that $0.7$ million parameters are sufficient.
The improvement is more significant for larger datasets. To memorize $15.8$ million bounding boxes in the Open Images V6 dataset\footnote{https://storage.googleapis.com/openimages/web/index.html} with $600$ classes, our result shows that only $4.5$ million parameters suffice.

For a large class (i.e., $2^{O(N^w)}$-separated) of datasets with $N$ pairs, Theorem~\ref{thm:lw-informal} improves the number of parameters sufficent for memorization from $O(N)$ down to $O(N^w)$ for any $w\in[2/3,1]$, by exploiting network depth that increases with the number of pairs $N$.
Then, it is natural to ask whether the depth increasing with $N$ is \emph{necessary} for memorization with a sub-linear number of parameters.
The following existing result on the VC-dimension implies that increasing depth is necessary for memorization with $o(N/\log N)$ parameters, at least for $\relu$ networks.
\begin{theorem*}[\cite{bartlett19}]\label{thm:reluvc}
{\bf(Informal)} For $L$-layer $\relu$ networks, $\Omega(N/(L\log N))$ parameters are necessary for memorizing at least one set of $N$ inputs with arbitrary labels.
\end{theorem*}
The above theorem implies that for $\relu$ networks of constant depth, $\Omega(N/\log N)$ parameters are necessary for memorizing at least one set of $N$ inputs with arbitrary labels.
In contrast, by increasing depth with $N$, Theorem~\ref{thm:lw-informal} shows that there is a large class of datasets that can be memorized with $o(N/\log N)$ parameters.
Combining these two results, one can conclude that increasing depth is \emph{necessary and sufficient} for memorizing a large class of $N$ pairs with $o(N/\log N)$ parameters.

Given that the depth is critical for memorization with $o(N/\log N)$ parameters, is the width also critical? We prove that it is not the case, via the following theorem.
\begin{customthm}{2}[Informal]\label{thm:3-informal}
For a fully-connected network of width $3$ using a sigmoidal or $\relu$ activation function, $O(N^{2/3}+\log\Delta)$ parameters (i.e., layers) suffice to memorize any $\Delta$-separated set of $N$ pairs.
\end{customthm}
Theorem~\ref{thm:3-informal} states that under $2^{O(N^{2/3})}$-separateness of inputs, the network width does not necessarily have to increase with $N$ for memorization with sub-linear parameters.
Furthermore, it shows that even a surprisingly narrow network of width $3$ has a superior memorization power than a fixed-depth network, requiring only $O(N^{2/3})$ parameters for memorizing any $2^{O(N^{2/3})}$-separated $N$ pairs.

Theorems~\ref{thm:lw-informal} and \ref{thm:3-informal} show that we can construct some \emph{tailor-made} network architectures that can memorize $N$ pairs with $o(N)$ parameters, under the $\Delta$-separateness condition. This means that these theorems do not answer the question of how many such data points can a \emph{given} network memorize.
To answer this question, we provide sufficient conditions for identifying the maximum number of points given general networks (Theorem~\ref{thm:criteria}).
In a nutshell, our conditions indicate that to memorize more pairs under the same budget for the number of parameters, the network must be deep, and most of its parameters should be concentrated in the first layers. In other words, the network must have a deep and narrow architectures in the final layers, as also proposed in \citep{bartlett19}.
Our sufficient conditions successfully incorporate the characteristics of datasets, the number of parameters, and the architecture, which enable us to memorize $\Delta$-separated datasets with number of pairs super-linear in the number of parameters.
This is in contrast to the prior results that the number of arbitrary pairs that can be memorized is at most proportional to the number of parameters \citep{yamasaki93,yun19,vershynin20}.

Finally, we provide empirical results corroborating our theoretical findings that deep networks often memorize better than their shallow counterparts with a similar number of parameters.
Here, we emphasize that better memorization power does not necessarily imply better generalization.
We indeed observe that shallow and wide networks often generalize better than deep and narrow networks, given the same (or similar) training accuracy. 

\paragraph{Organization.} We first introduce related works in Section~\ref{sec:relatedworks}. In Section~\ref{sec:preliminary}, we introduce notation and the problem setup.
We formally state our main results and discuss them in Section~\ref{sec:mainresults}. We provide the proof of our main theorem in Section~\ref{sec:pfsketch:lw}.
In Section~\ref{sec:exp}, we provide empirical observations on the effect of depth and width in neural networks.
Finally, we conclude the paper in Section~\ref{sec:conclusion}.

\section{Related works}\label{sec:relatedworks}
While introducing related works, we hide the input dimension $d_x$ in $O,\Omega,\Theta$ as in convention. We note that in our main results in Section \ref{sec:mainresults}, $O,\Omega,\Theta$ will only hide universal constants, excluding $d_x$.
\paragraph{Sufficient number of parameters for memorization.} Identifying the sufficient number of parameters for memorizing arbitrary $N$ pairs has a long history.
Earlier works mostly focused on bounding the number of hidden neurons in shallow networks required for memorization.
\citet{baum88} proved that for $2$-layer $\step$\footnote{$\step$ denotes the binary threshold activation function: $x\mapsto\mathbf{1}[x\ge0]$.} networks, $O(N)$ hidden neurons (i.e., $O(N)$ parameters) are sufficient for memorizing arbitrary $N$ pairs when inputs are in general position.
\citet{huang98} showed that the same bound holds for any bounded and nonlinear activation function $\sigma$ satisfying that either $\lim_{x\rightarrow-\infty}\sigma(x)$ or $\lim_{x\rightarrow\infty}\sigma(x)$ exists, without any condition on inputs.
The $O(N)$ bounds on the number of hidden neurons was improved to $O(\sqrt{N})$ by exploiting an additional hidden layer by
\cite{huang03}; nevertheless, this construction still requires $O(N)$ parameters.

With the advent of deep learning,
the focus has shifted to modern activation functions and deeper architectures. 
\cite{zhang17} proved that $O(N)$ hidden neurons are sufficient for $2$-layer $\relu$ networks to memorize arbitrary $N$ pairs. 
\cite{yun19} showed that for deep $\relu$ (or hard $\tanh$) networks having at least $3$ layers, $O({N})$ parameters are sufficient.
\cite{vershynin20} proved a similar result for $\step$ (or $\relu$) networks for memorizing arbitrary set of unit vectors $\{x_i\}_{i=1}^N$ satisfying $\|x_i-x_j\|_2^2=\Omega(\frac{\log\log d_{\max}}{\log  d_{\min}})$, $N=e^{O(d_{\min}^{1/5})}$, and $d_{\max}=e^{O(d_{\min}^{1/5})}$ where $d_{\max}$ and $d_{\min}$ denote the maximum and the minimum hidden dimensions, respectively.

In addition, the memorization power of modern network architectures has also been studied. 
\cite{hardt17} showed that  $\relu$ networks consisting of residual blocks with $O(N)$ hidden neurons can memorize any arbitrary set of unit vectors $\{x_i\}_{i=1}^N$ satisfying $\|x_i-x_j\|_2\ge\rho$ for some absolute constant $\rho>0$.
\cite{nguyen18} studied a broader class of layers and proved that $O(N)$ hidden neurons suffice for convolutional neural networks consisting of fully-connected, convolutional, and max-pooling layers for memorizing arbitrary $N$ pairs having different patches.

\paragraph{Necessary number of parameters for memorization.} The necessary number of parameters for memorization has also been studied. \cite{sontag97} showed that for any neural network using analytic definable activation functions, $\Omega(N)$ parameters are necessary for memorizing arbitrary $N$ pairs. Namely, given any network using analytic definable activation with $o(N)$ parameters, there exists a set of $N$ pairs that the network cannot memorize.

The Vapnik-Chervonenkis (VC) dimension is also closely related to the memorization power of neural networks.
While the memorization power studies the number of parameters for memorizing \emph{arbitrary} $N$ pairs, the VC-dimension is related to the number of parameters for memorizing \emph{at least one set} of $N$ inputs with arbitrary labels. In particular, upper bounds on the VC-dimension translates to lower bounds on the necessary number of parameters for memorizing arbitrary $N$ pairs.
The VC-dimension of neural networks has been studied for various types of activation functions. 
For memorizing at least one set of $N$ inputs with arbitrary labels, it is known that $\Theta(N/\log N)$ parameters are necessary \citep{baum89} and sufficient \citep{maass97} for $\step$ networks.
Similarly,
\cite{karpinski97} proved that $\Omega(\sqrt{N}/U)$ parameters are necessary for sigmoid networks of $U$ neurons.
Recently, \cite{bartlett19} showed that $\Theta(N/(\bar L\log N))$ parameters are necessary and sufficient for $L$-layer networks using any piecewise linear activation function where $\bar L:=\frac{1}{W_L}\sum_{\ell=1}^LW_\ell$ and $W_\ell$ denotes the number of parameters up to the $\ell$-th layer.
\paragraph{Benefits of depth in neural networks}
To understand deep learning, researchers have investigated the advantages of deep neural networks compared to shallow neural networks with a similar number of parameters.
Initial results discovered examples of deep neural networks that cannot be approximated by shallow neural networks without using exponentially many parameters \citep{telgarsky16,eldan16,arora18}.
Recently, it is discovered that deep neural networks require fewer parameters than shallow neural networks to represent or approximate a class of periodic functions \citep{chatziafratis20b,chatziafratis20a}.
For approximating continuous functions, \cite{yarotsky18} proved that the number of required parameters for $\relu$ networks of constantly bounded depth are square to that for deep $\relu$ networks.

\section{Problem setup}\label{sec:preliminary}
In this section, we describe the problem setup and notation. First, we introduce frequently used notation.
We use $\log$ to denote the logarithm to the base $2$. We let $\relu$ be the function $x\mapsto\max\{x,0\}$, $\step$ be the function $x\mapsto\mathbf{1}[x\ge0]$, and $\id$ be the function $x\mapsto x$.
For $\mathcal X\subset\mathbb R$, we denote $\lfloor\mathcal X\rfloor:=\{\lfloor x\rfloor:x\in\mathcal X\}$.
For $n\in\mathbb N$ and a set $\mathcal X$, we denote $\binom{\mathcal X}{n}:=\{\mathcal S\subset\mathcal X:|\mathcal S|=n\}$. For $n\in\mathbb N$, we use $[n]:=\{0,\dots,n-1\}$. We also extend the modulo operation to $x\bmod y:=x-y\cdot\lfloor\frac{x}{y}\rfloor$ for $x\ge0$ and $y>0$. 
To describe our results precisely, we use $O(\cdot),\Theta(\cdot),\Omega(\cdot)$ in our theorems and lemmas only to hide universal constants, excluding the input dimension $d_x$.

Throughout this paper, we consider fully-connected feedforward networks. In particular, we consider the following setup: Given an activation function $\sigma$, we define a neural network $f_\theta$ of $L$ layers (or equivalently $L-1$ hidden layers), input dimension $d_x$, output dimension $1$, and hidden layer dimensions $d_1,\dots,d_{L-1}$ parameterized by $\theta$ as $f_\theta:=t_{L-1}\circ\sigma\circ\cdots\circ t_1\circ\sigma\circ t_0.$ 
Here, $t_\ell:\mathbb{R}^{d_\ell}\rightarrow\mathbb R^{d_{\ell+1}}$ is an affine transformation parameterized by $\theta$.\footnote{We set $d_0:=d_x$ and $d_{L}:=1$.} We count the number of parameters as the number of weights and biases in networks. We define the width of $f_\theta$ as the maximum over $d_1,...,d_{L-1}$.
We denote a neural network using an activation function $\sigma$ by a ``$\sigma$ network''
and a neural network using two activation functions $\sigma_1,\sigma_2$ by a ``$\sigma_1$+$\sigma_2$ network.''

As we introduced in Section \ref{sec:contribution}, our main results hold for any sigmoidal activation function and $\relu$. 
Formally, we define the \emph{sigmoidal} functions as follows. 
\begin{mydefinition}\label{def:sigmoidal}
We say a function $\sigma:\mathbb{R}\rightarrow\mathbb{R}$ is sigmoidal if the following conditions hold.
\begin{list}{{\tiny$\bullet$}}{\leftmargin=1.8em}
  \setlength{\itemsep}{1pt}
  \vspace*{-4pt}
    \item
Both $\lim_{x\rightarrow-\infty}\sigma(x)$, $\lim_{x\rightarrow\infty}\sigma(x)$ exist and $\lim_{x\rightarrow-\infty}\sigma(x)\ne\lim_{x\rightarrow\infty}\sigma(x)$.
\item There exists $z\in\mathbb R$ such that $\sigma$ is continuously differentiable at $z$ and $\sigma^\prime(z )\ne0$.
\end{list}
\end{mydefinition}
A class of sigmoidal functions covers many activation functions including sigmoid, $\tanh$, hard $\tanh$, etc. 
Furthermore, since hard $\tanh$ can be represented as a composition or a summation of two $\relu$ functions,\footnote{hard $\tanh(x)=\relu(x+1)-\relu(x-1)-1=\relu\big(2-\relu(1-x)\big)-1$} our results for sigmoidal activation functions hold for $\relu$ as well, with an additional constant factor on depth or width.

Lastly, we formally define the memorization as follows.
\begin{mydefinition}\label{def:memorization}
Given $C,d_x\in\mathbb N$, a set of inputs $\mathcal X\subset\mathbb R^{d_x}$, a label function $y:\mathcal X\rightarrow[C]$, and a network $f_\theta$ parameterized by $\theta$, we say $f_\theta$ can memorize $\{(x,y(x))\}_{x\in\mathcal X}$ in $d_x$ dimension with $C$ classes
if for any $\varepsilon>0$, there exists $\theta$ such that $|f_\theta(x)-y(x)|\le\varepsilon$ for all $x\in\mathcal X$.
\end{mydefinition}
In Definition \ref{def:memorization}, we define memorizability as the ability to uniformly \textit{approximate} the labels up to an arbitrarily small error $\varepsilon > 0$, instead of the exact memorization (i.e., $\varepsilon = 0$). 
This \textit{almost lossless} definition allows us to use the following two-step proof strategy: (1) Construct a $\step+\id$ network that exactly memorizes the dataset, (2) Approximate the $\step+\id$ network by another network using \textit{any} sigmoidal activation function; see Section~\ref{sec:pfsketch:lw} for a more detailed outline. Nevertheless, the exact memorization holds for the activation functions where an exact implementation of $\step+\id$ network for finite inputs is possible, e.g., hard $\tanh$ or $\relu$. We often write ``$f_\theta$ can memorize arbitrary $N$ pairs'' without ``in $d_x$ dimension with $C$ classes'' throughout the paper.

\section{Main results}\label{sec:mainresults}
\subsection{Memorization via sub-linear parameters}\label{sec:sublinear}
\paragraph{Efficacy of depth for memorization.} Now, we introduce our main theorem on memorizing $N$ pairs with $o(N)$ parameters.
The proof of Theorem \ref{thm:lw} is presented in Section~\ref{sec:pfsketch:lw}.
\begin{theorem}\label{thm:lw} For any $C, N, d_x \in \mathbb{N}$, $w\in[2/3,1]$, and a sigmoidal activation function $\sigma$, there exists a $\sigma$ network $f_\theta$ of $O\big(\log  d_x+\log\Delta+\frac{N^{2-2w}}{1+(1.5w-1)\log N}\log C\big)$ hidden layers and {$O\big(d_x+\log\Delta+N^w+{N^{1-w/2}}\log C\big)$} parameters such that $f_\theta$ can memorize any $\Delta$-separated set of $N$ pairs in $d_x$ dimension with $C$ classes.
\end{theorem}
Note that the $1+(1.5w-1)\log N$ denominator exists in the number of layers in Theorem \ref{thm:lw}, which is omitted in its informal version in Section \ref{sec:contribution}.
In addition, while we only address sigmoidal activation functions in the statement of Theorem~\ref{thm:lw}, the same conclusion holds for $\relu$ networks as we described in Section \ref{sec:preliminary}.

In Theorem \ref{thm:lw}, $\Delta$ induces only $O(\log\Delta)$ overhead  to the number of layers and the number of parameters.
As we introduced in Section~\ref{sec:contribution}, $\log\Delta$ for modern datasets is often very small.
Furthermore, $\log\Delta$ can be small for random inputs. For example, a set of $d_x$-dimensional i.i.d.\ standard normal random vectors of size $N$ satisfies $\log\Delta=O(\frac1{d_x}\log(N/\sqrt{\delta}))$ with probability at least $1-\delta$ (see Lemma \ref{lem:gaussian} in Appendix \ref{sec:gaussian}).
Hence, the $\Delta$-separateness condition is often negligible.

Suppose that $d_x$ and $C$ are treated as constants, as also assumed in existing results. Then, Theorem~\ref{thm:lw} implies that if $\log\Delta=O(N^w)$ for some $w<1$, then $\Theta(N^w)$ (i.e., sub-linear to $N$) parameters are \emph{sufficient} for networks using a sigmoidal or $\relu$ activation function to memorize arbitrary $\Delta$-separated set of $N$ pairs.
Note that the condition $\log\Delta=O({N^w})$ is very loose for many practical datasets, especially for those with huge $N$.
Combined with the lower bound $\Omega(N/\log N)$ on the \emph{necessary} number of parameters for $\relu$ networks of constant depth \citep{bartlett19},
Theorem~\ref{thm:lw} implies that the depth growing in $N$ is \emph{necessary and sufficient} for memorizing a large class (i.e., $\Delta$-separated) of $N$ pairs with $o(N/\log N)$ parameters. In other words, deep $\relu$ networks have stronger memorization power than shallow $\relu$ networks.

\paragraph{Unimportance of width for memorization.} Given that the depth is critical for memorization with $o(N/\log N)$ parameters, we show that the width is not very critical.
Specifically, we prove that extremely narrow networks of width $3$ can memorize with $O(N^{2/3})$ layers (i.e., $O(N^{2/3})$ parameters) under $\log\Delta=O(N^{2/3})$ and constant $d_x$, as stated in the following theorem.
The proof of Theorem \ref{thm:3} is presented in Section \ref{sec:criteria}.
\begin{theorem}\label{thm:3} For any $C,N,d_x\in\mathbb{N}$, $\Delta\ge 1$, and sigmoidal activation function $\sigma$,
a $\sigma$~network of $\Theta(\log d_x+\log\Delta+N^{2/3}\log C)$ hidden layers and width $3$ can memorize any $\Delta$-separated set of $N$ pairs in $d_x$ dimension with $C$ classes.
\end{theorem}

The statement of Theorem \ref{thm:3} might be somewhat surprising since the network width for memorization does not depend on the input dimension $d_x$.
This is in contrast with the recent universal approximation results that width at least $d_x+1$ is necessary for approximating functions on $d_x$-dimensional domains  \citep{lu17,hanin17,johnson19,park21}. 
The main difference follows from the fundamental difference in the two approximation problems, i.e., approximating a function at finite inputs versus infinite inputs (e.g., the unit cube).
Any set of $N$ distinct input vectors can be easily mapped to $N$ different scalar values by taking inner products with a random vector (e.g., standard Gaussian). Hence, memorizing finite input-label pairs in $d_x$-dimension can be easily translated into memorizing finite input-label pairs in \emph{one}-dimension. In other words, the input dimensionality $d_x$ is not very important as long as they can be translated to distinct scalar values. 
In contrast, there is no ``natural'' way to design an injection from $d_x$-dimensional unit cube to lower dimension. Namely, to not lose the ``information'' from inputs, width $d_x$ is required.
Therefore, an arbitrary function cannot be approximated by the network of width independent of $d_x$.

\subsection{Sufficient conditions for identifying memorization power}\label{sec:criteria}
While Theorem \ref{thm:lw} constructs some \emph{tailor-made} network architectures for memorization, the following theorem states sufficient conditions for verifying the memorization power of \emph{any given} network architectures. The proof of Theorem \ref{thm:criteria} is presented in Appendix \ref{sec:pfthm:criteria}.

\begin{theorem}\label{thm:criteria}
For any sigmoidal activation function $\sigma$, let $f_\theta$ be a $\sigma$ network of $L$ hidden layers having $d_\ell\ge3$ neurons at the $\ell$-th hidden layer.
Then, for any $C,N,d_x\in\mathbb N$ and $\Delta>1$, $f_\theta$ can memorize any $\Delta$-separated set of $N$ pairs in $d_x$ dimension with $C$ classes if the following statement holds:

There exist $0<L_1<\dots<L_K<L$ for some $2\le K\le\log N$ satisfying the conditions below.
\begin{list}{{\tiny$\bullet$}}{\leftmargin=1.8em}
  \setlength{\itemsep}{1pt}
  \vspace*{-4pt}
    \item[1.] $\prod_{\ell=1}^{L_1}\lfloor(d_\ell+1)/2\rfloor\ge\Delta\sqrt{2\pi d_x}$.
    \item[2.] $\sum_{\ell=L_{i-1}+1}^{L_{i}}(d_\ell-2)\ge2^{i+3}$ for all $1<i\le K-1$.
    \item[3.] $2^{K}\cdot\Big(\sum_{\ell=L_{K-1}+1}^{L_K}(d_\ell-2)\Big)\cdot\Big\lfloor\frac{L-L_K}{2\lceil\log C\rceil+1}\Big\rfloor\ge N^2+4$.
\end{list}
\end{theorem}
Our conditions in Theorem \ref{thm:criteria} require that the layers of the network can be ``partitioned'' into $K+1$ distinct parts characterized by $L_1,\dots,L_K$ for some $K\ge2$.
The partition into $K+1$ parts corresponds to conditions in Theorem \ref{thm:criteria}: The first part is for the first condition, the second to the $(K-1)$-th parts are for the second condition, and the remaining parts are for the third. We postpone the discussion on the role of each condition to Section \ref{sec:pfsketch:lw}. 
In Theorem \ref{thm:criteria}, one can observe that the choice of $K$ affects the second condition and the third condition. In particular, larger $K$ implies that more hidden neurons are required to satisfy the second condition but less hidden neurons are required to satisfy the third condition.
We discuss how $K$ affects these two conditions in more detail later in this section.

Now we discuss the three conditions in Theorem \ref{thm:criteria} in detail.
The first condition considers the first $L_1$ hidden layers.
In order to satisfy this condition, deep and narrow architectures are better than shallow and wide architectures under a similar number of parameters due to the product form $\prod_{\ell=1}^{L_1}\lfloor(d_\ell+1)/2\rfloor$.
Nevertheless, the architecture of the first $L_1$ hidden layers is not very critical since only $\log(\Delta\sqrt{2\pi d_x})$ layers are sufficient even with width $3$ (e.g.,  $\log\Delta<17$ for the ImageNet dataset).

The second condition requires that the number of hidden neurons $\sum_{\ell = L_{i-1}+1}^{L_i} (d_\ell - 2)$ in each \emph{part} of layers is larger than or equal to some threshold value $2^{i+3}$.
In particular, the second condition is closely related to the third condition due to the following trade-off: As $K$ increases, the LHS in the third condition increases, i.e., increasing $K$ by one decreases the required neurons/layers after the $L_{K-1}$-th layer for satisfying the third condition by half. However, the second condition requires more hidden neurons as $K$ grows.
Simply put, increasing $K$ by one requires doubling hidden neurons from the $(L_1+1)$-th hidden layer to the $L_{K-1}$-th hidden layer.
Nevertheless, this doubling hidden neurons makes the LHS of the third condition double as well. 
In other words, due to the trade-off between the second condition and the third condition, there will be some optimal choice of $K$ (and $L_1,\dots,L_K$) maximizing the number of memorizable pairs for a given network.

The third condition is simple. As we explained, $2^K$ is approximately proportional to the number of hidden neurons from the $(L_1+1)$-th hidden layer to the $L_{K-1}$-th hidden layer.
The second term in the LHS of the third condition $\sum_{\ell=L_{K-1}+1}^{L_K}(d_\ell-2)$ simply counts the hidden neurons in the $K$-th part. 
On the other hand, the last term counts the number of layers in the $(K+1)$-th part. This indicates that to satisfy conditions in Theorem \ref{thm:criteria} using few parameters, the last layers of the network should be deep and narrow. In particular, we note that such a deep and narrow architecture in the last layers is indeed necessary for $\relu$ networks to memorize with $o(N/\log N)$ parameters \citep{bartlett19}, as we discussed in Section \ref{sec:contribution}. 

\paragraph{Deriving Theorem \ref{thm:3} using Theorem \ref{thm:criteria}.} Now, we describe how to show memorization with a sub-linear number of parameters using our sufficient conditions by deriving Theorem \ref{thm:3} from Theorem \ref{thm:criteria}. 
Consider a network of width~$3$, i.e., $d_\ell=3$ for all $\ell$.
As we explained, the first conditions can be easily satisfied using $\Theta(\log d_x+\log\Delta)$ layers. For the second condition, consider choosing $K=\log(N^{2/3})$, then $\Theta(N^{2/3})$ hidden neurons (i.e., $\Theta(N^{2/3})$ hidden layers) would be sufficient, i.e., $L_{K-1}-L_1=\Theta(N^{2/3})$.
Finally, we choose $L_K$ and $L$ to satisfy $L_{K}-L_{K-1}=\Theta(N^{2/3})$ and $L-L_K=\Theta(N^{2/3}\log C)$.
Then, it naturally satisfies the third condition using only $\Theta(N^{2/3}\log C)$ layers and completes the derivation of Theorem \ref{thm:3}.

\subsection{Discussions on Theorems \ref{thm:lw}--\ref{thm:criteria}}
\paragraph{Extension to regression problem.} The results of Theorems \ref{thm:lw}--\ref{thm:criteria} can be easily applied to the regression problem, i.e., when labels are from $[0,1]$.
This is because one can simply translate the regression problem with some $\varepsilon>0$ error tolerance to the classification problem with $\lceil1/\varepsilon\rceil$ classes. Here, each class $c\in\{0,1,\dots,\lceil1/\varepsilon\rceil-1\}$ corresponds to the target value $c\cdot\varepsilon$.
Hence, the regression problem can also be solved within $\varepsilon$ error with $o(N)$ parameters, where the sufficient number of layers and the sufficient number of parameters are identical to the numbers in Theorems~\ref{thm:lw}--\ref{thm:criteria} with the replacement of $\log C$ with $\log(1/\varepsilon)$.

\paragraph{Relation with benefits of depth in neural networks.} 
Our observation that deeper $\relu$ networks have more memorization power is closely related to the recent studies on the benefits of depth in neural networks (see Section \ref{sec:relatedworks}).
While our observation indicates that depth is critical for the memorization power, these works mostly focused on showing the importance of depth for approximating functions.
Here, the existing results on the benefits of depth for function approximation cannot directly imply the benefits of depth for memorization since they often focus on specific classes of functions or require parameters far more than $O(N)$.

\paragraph{On parameter precision.}
As in many works on the expressive power of neural networks \citep{cybenko89,pinkus99,lu17,bartlett19,yun19,park21}, we also assume real-valued parameters (i.e., weights and biases) rather than parameters of constantly bounded precision (e.g., binary, 16 bit floating-point numbers).
Notably, parameters of constantly bounded precision is provably \emph{insufficient} for memorization with $o(N)$ parameters: Any network of $o(N)$ parameters of constantly bounded precision cannot memorize any set of $N$ inputs with arbitrary labels \citep{shalev14}. Hence, parameters of $\omega(1)$ precision (e.g., real-valued parameters) are \emph{necessary} for memorization with $o(N)$ parameters.
We note that the maximum required precision for Theorem \ref{thm:lw}  is $O(N^{2/3})$ at $w=2/3$, i.e., memorization with $O(N^{2/3})$ parameters. This required precision decreases to $\Theta(1)$ as $w$ goes to $1$.

\section{Network construction for proof of Theorem \ref{thm:lw}}\label{sec:pfsketch:lw} 
{In this section, we give a constructive proof of Theorem \ref{thm:lw}. First, we sketch the main proof idea in Section \ref{sec:pfsketch}. Then, we provide the formal proof of Theorem \ref{thm:lw} in Section~\ref{sec:network_construction}. 
Lastly, we describe connections between our network construction and conditions in Theorem \ref{thm:criteria} in Section~\ref{sec:func_criteria}.

\subsection{Proof outline}\label{sec:pfsketch}
Before introducing the formal proof, we first sketch our network construction.
In our proof, we mainly focus on constructing a $\step+\id$ network\footnote{Recall from Section \ref{sec:preliminary} that $\step:x\mapsto\mathbf1[x\ge0]$ and $\id:x\mapsto x$.}  which can memorize any $\Delta$-separated $N$ pairs in $d_x$ dimension with $C$ classes using
$O\big(\log  d_x+\log\Delta+\frac{N^{2-2w}}{1+(1.5w-1)\log N}\log C\big)$ hidden layers and {$O\big(d_x+\log\Delta+N^w+{N^{1-w/2}}\log C\big)$} parameters, instead of designing a target $\sigma$ network directly.
Then, we approximate this $\step+\id$ network using a $\sigma$ network of the same architecture to complete the proof. 

We first briefly discuss the motivation of our construction which can memorize $N$ pairs with $o(N)$ parameters.
Our construction is motivated by the $\relu$ network construction achieving nearly tight VC-dimension \citep{bartlett19}: We observe that there exists a network of $o(N)$ parameters that can memorize any \emph{well-separated} $N$ scalar values in $\big[0,o(N^2)\big)$ with arbitrary labels
(Lemma \ref{lem:memorizer-rough} in Section \ref{sec:network_construction}).
Here, the \emph{well-separateness} means that each scalar value falls into each interval $[i,i+1)$.
Namely, once we map input vectors to well-separated values in $\big[0,o(N^2)\big)$ using $o(N)$ parameters, then we can memorize these input vectors with arbitrary labels with $o(N)$ parameters.

Under our motivation, we now describe a high-level outline of our $\step+\id$ network construction.
First, we project $\Delta$-separated, $d_x$-dimensional input vectors to well-separated scalar values in some bounded interval $\big[0,O(N^2\sqrt{d_x}\Delta)\big)$, using $O(1)$ hidden layers and $O(d_x)$ parameters. 
In the following layers, we gradually decrease the upper bound on these scalar values from $O(N^2\sqrt{d_x}\Delta)$ to $O(N^{2-w})$ (i.e., $o(N^2)$) by relocating these values while preserving their well-separateness. This step utilizes $O(\log \Delta+\log d_x+\log N)$ hidden layers and $O(\log \Delta+\log d_x+N^w)$ parameters. Here, note that under the well-separateness, there exists one-to-one correspondence between these scalar values and input vectors. 
Finally, we map these scalar values in $\big[0,O(N^{2-w})\big)$ to corresponding labels using
$O\big(\frac{N^{2-2w}}{1+(1.5w-1)\log N}\log C\big)$ hidden layers and $O(N^w+{N^{1-w/2}}\log C)$ parameters. 
By approximating this $\step+\id$ network construction using a $\sigma$ network, we obtain the statement of Theorem \ref{thm:lw}.
The overall network construction is illustrated in Figure \ref{fig:pfsketch}.

\begin{figure}
    \centering
    \includegraphics[width=\textwidth]{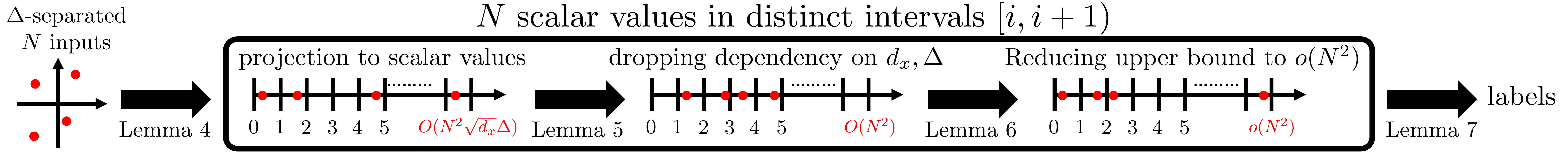}
    \caption{Illustration of the network construction for the proof of Theorem \ref{thm:lw}.}
    \label{fig:pfsketch}
\end{figure}

\subsection{Constructing target network}\label{sec:network_construction}
Now, we formally describe our network construction.
To this end, we first construct the network using $\step$ and $\id$ activation functions and approximate this network using a $\sigma$ network.
\paragraph{Projecting input vectors to scalar values.} We first map input vectors to bounded well-separated scalar values using the following lemma. The proof of Lemma \ref{lem:projection-rough} is presented in Appendix \ref{sec:pflem:projection-rough}.
\begin{lemma}\label{lem:projection-rough}
For any $\Delta$-separated $\mathcal X\in\binom{\mathbb R^{d_x}}{N}$, there exist $v\in\mathbb R^{d_x}$, $b\in\mathbb R$ such that $$\big\lfloor\{v^\top x+b:x\in\mathcal X\}\big\rfloor\in\binom{[O(N^2\sqrt{d_x}\Delta)]}{N}.$$
\end{lemma}
Lemma \ref{lem:projection-rough} states that any $\Delta$-separated $N$ vectors can be mapped to some well-separated scalar values bounded by $O(N^2\sqrt{d_x}\Delta)$ using a simple projection with $O(1)$ hidden layers and $O(d_x)$ parameters. 
\paragraph{Decreasing upper bound on scalar values to $O(N^{2-w})$.} We map these well-separated $N$ values upper bounded by $O(N^2\sqrt{d_x}\Delta)$ to another well-separated values of the smaller upper bound $\lfloor{N^2}/{4}+1\rfloor$ using the following lemma. The proof of Lemma \ref{lem:compression1-rough} is presented in Appendix \ref{sec:pflem:compression1-rough}.
\begin{lemma}\label{lem:compression1-rough}
For any $\mathcal X\in\binom{\mathbb R}{N}$ and $K\in\mathbb N$ such that $\lfloor\mathcal X\rfloor\in\binom{[K]}{N}$, there exists a $\step+\id$ network $f$ of $1$ hidden layer and width $3$ such that $\lfloor f(\mathcal X)\rfloor\in\binom{[T]}{N}$ where $T:=\max\big\{\lceil{K}/{2}\rceil,\lfloor{N^2}/{4}+1\rfloor\big\}$.
\end{lemma}
Lemma \ref{lem:compression1-rough} states that $O(1)$ hidden layers and $O(1)$ parameters are sufficient for halving the upper bound until $\lfloor {N^2}/4+1\rfloor$.
Therefore, by applying Lemma \ref{lem:compression1-rough} $O(\log\Delta+\log d_x)$ times, a network of $O(\log\Delta+\log d_x)$ hidden layers and $O(\log\Delta+\log d_x)$ parameters can decrease the upper bound $O(N^2\sqrt{d_x}\Delta)$ to $\lfloor {N^2}/4+1\rfloor$. 
The intuition behind Lemma \ref{lem:compression1-rough} is that if  the number of target intervals $T$ is large enough compared to the number of inputs $N$, then inputs can be easily mapped without harming the well-separateness (i.e., no two inputs are mapped into the same interval $[i,i+1)$) using some simple network of a small number of layers and parameters.
In particular, we construct the simple network $f$ in the proof of Lemma \ref{lem:compression1-rough} as 
\begin{align}
    f(x):=\begin{cases}
    x~&\text{if}~x\in\big[0,T\big)\\
    (x+b)\bmod T~&\text{if}~x\in\big[T,K\big)
    \end{cases}\label{eq:compression1-rough-f}
\end{align}
for some $b\in[T]$. 
However, if $T$ is not large enough compared to $N$ (i.e., $T<\lfloor N^2/4+1\rfloor)$, then outputs of our network \eqref{eq:compression1-rough-f} may not preserve the well-separateness, i.e., $b\in[T]$ satisfying $f(\lfloor\mathcal X\rfloor)\in\binom{[T]}{N}$ may not exist. 
However, the upper bound on the well-separated values can be further decreased by utilizing more parameters, as stated in the following lemma. The proof of Lemma \ref{lem:compression2-rough} is presented in Section \ref{sec:pflem:compression2-rough}.

\begin{lemma}\label{lem:compression2-rough}
For any $\mathcal X\in\binom{\mathbb R}{N}$ and $K\in\mathbb N$ such that $\lfloor\mathcal X\rfloor\in\binom{[K]}{N}$, there exists a $\step+\id$ network $f$ of $1$ hidden layer and width $O\big({N^2}/{K}\big)$ such that $\lfloor f(\mathcal X)\rfloor\in\binom{[T]}{N}$ where $T:=\max\big\{\lceil{K}/{2}\rceil,N\big\}$.
\end{lemma}
The network in Lemma \ref{lem:compression2-rough} can halve the upper bound $K$ of the inputs beyond $\lfloor N^2/4+1\rfloor$. However, the required number of parameters will be doubled if the current upper bound $K$ decreases by half.
Hence, in order to decrease the upper bound from $\lfloor N^2/4+1\rfloor$ to $O(N^{2-w})$ using $O(\log N)$ applications of Lemma \ref{lem:compression2-rough}, we need $O(\log N)$ hidden layers and $O(N^w)$ parameters. Here, we construct each application of Lemma \ref{lem:compression2-rough} using two hidden layers. One hidden layer of $O(N^2/K)$ hidden neurons implements $f$ in Lemma \ref{lem:compression2-rough}. The other hidden layer of one hidden neuron just bypasses the one-dimensional output of $f$ implemented in the previous layer. The reason behind using an extra layer is that, if we implement each application of Lemma \ref{lem:compression2-rough} using a single hidden layer (as we did for Lemma \ref{lem:compression1-rough}), then it will introduce $O(N^4/K^2)$ parameters between each pair of adjacent hidden layers of width $O(N^2/K)$; we avoid these unnecessary parameters by adding extra layers of width 1.

\paragraph{Mapping scalar values to corresponding labels.}
So far, we have $N$ well-separated values, upper bounded by $O(N^{2-w})$. Now, we map these values to corresponding labels (i.e., labels of their original inputs) using the following lemma, motivated by the $\relu$ network achieving nearly tight VC-dimension \citep{bartlett19}. The proof of Lemma \ref{lem:memorizer-rough} is presented in Appendix~\ref{sec:pflem:memorizer-rough}.
\begin{lemma}\label{lem:memorizer-rough}
For any $C,V\in\mathbb N$ and $p\in[1/2,1]$, there exists a $\step+\id$ network $f_\theta$ of $O(V^p+V^{1/2}\log C)$ parameters and {$O\big(\frac{V^{1-p}}{1+(p-0.5)\log V}\big)$} hidden layers such that $f_\theta$ can memorize any $\mathcal X\in\binom{\mathbb R}{V}$ satisfying $\lfloor \mathcal X\rfloor=[V]$ with arbitrary labels in $[C]$.
\end{lemma}
If we set $w\in[2/3,1]$, $V=O(N^{2-w})$, and $p=\frac{w}{2-w}$ in Lemma \ref{lem:memorizer-rough}, then $O\big(\frac{N^{2-2w}}{1+(1.5w-1)\log N}\log C\big)$ hidden layers and $O(N^w+{N^{1-w/2}}\log C)$ parameters are sufficient for mapping any $N$ well-separated values in $\big[0,O(N^{2-w})\big)$ to corresponding labels.

Overall, our $\step+\id$ network construction 
mapping $\Delta$-separated, $d_x$-dimensional vectors to well-separated values bounded by $O(N^{2-w})$
requires $O\big(\log  d_x+\log\Delta+\frac{N^{2-2w}}{1+(1.5w-1)\log N}\log C\big)$ hidden layers and $O\big(d_x+\log\Delta+N^w+{N^{1-w/2}}\log C\big)$ parameters
by Lemmas \ref{lem:projection-rough}--\ref{lem:memorizer-rough}.
\paragraph{Approximating $\step+\id$ network using $\sigma$ network.}
Finally, we approximate our $\step+\id$ network construction by a $\sigma$ network of the same architecture, within any approximation error using the following lemma. 
The proof of Lemma \ref{lem:transform} is presented in Appendix~\ref{sec:pflem:transform}.
\begin{lemma}\label{lem:transform}
For any finite set of inputs $\mathcal X$, $\varepsilon>0$, $\step+\id$ network $f$, and sigmoidal activation function $\sigma$, there exists a $\sigma$ network $g$ having the same architecture with $f$ satisfying $|f(x)-g(x)|<\varepsilon$ for all $x\in\mathcal X.$
\end{lemma}
} 
\subsection{Functions of conditions in Theorem \ref{thm:criteria}}\label{sec:func_criteria}
The three conditions in Theorem \ref{thm:criteria} correspond to network constructions in Section \ref{sec:network_construction}. The first condition in Theorem \ref{thm:criteria} is for mapping $\Delta$-separated inputs to some well-separated values bounded by $O(N^2)$. 
Here, decreasing the upper bound on the well-separated set by half does not require a large number of layers and parameters until the upper bound reaches $\lfloor N^2/4+1\rfloor$ as stated in Lemma \ref{lem:compression1-rough}.
The second condition in Theorem \ref{thm:criteria} is for improving $O(N^2)$ upper bound on the well-separated values to $O(N^2/2^K)$. As stated in Lemma \ref{lem:compression2-rough}, decreasing the upper bound on the well-separated set by half (i.e., increasing $K$ by one in the second condition) requires {twice as many} parameters in the second condition.
Finally, the third condition in Theorem \ref{thm:criteria} is for mapping well-separated values to their labels, which corresponds to Lemma \ref{lem:memorizer-rough}. We refer Appendix \ref{sec:pfthm:criteria} for more detailed descriptions.
\section{Experiments}\label{sec:exp}
In this section, we study the effect of depth and width through experiments.
In particular, we empirically verify whether our theoretical finding extends to practices:  Do deep and narrow networks have more memorization power (i.e., better training accuracy) compared to their shallow and wide counterparts under similar numbers of parameters?
For the experiments, we use residual networks \citep{he16} having the same number of channels for each layer. The detailed experimental setups are presented in Appendix \ref{sec:expsetup}.
In our experiments, we report the training and test accuracy of networks by varying the number of channels and the number of residual blocks.
\begin{figure}[t]
\centering
\subfigure[CIFAR-10\label{fig:similar-cifar}]{\includegraphics[width=2.2in]{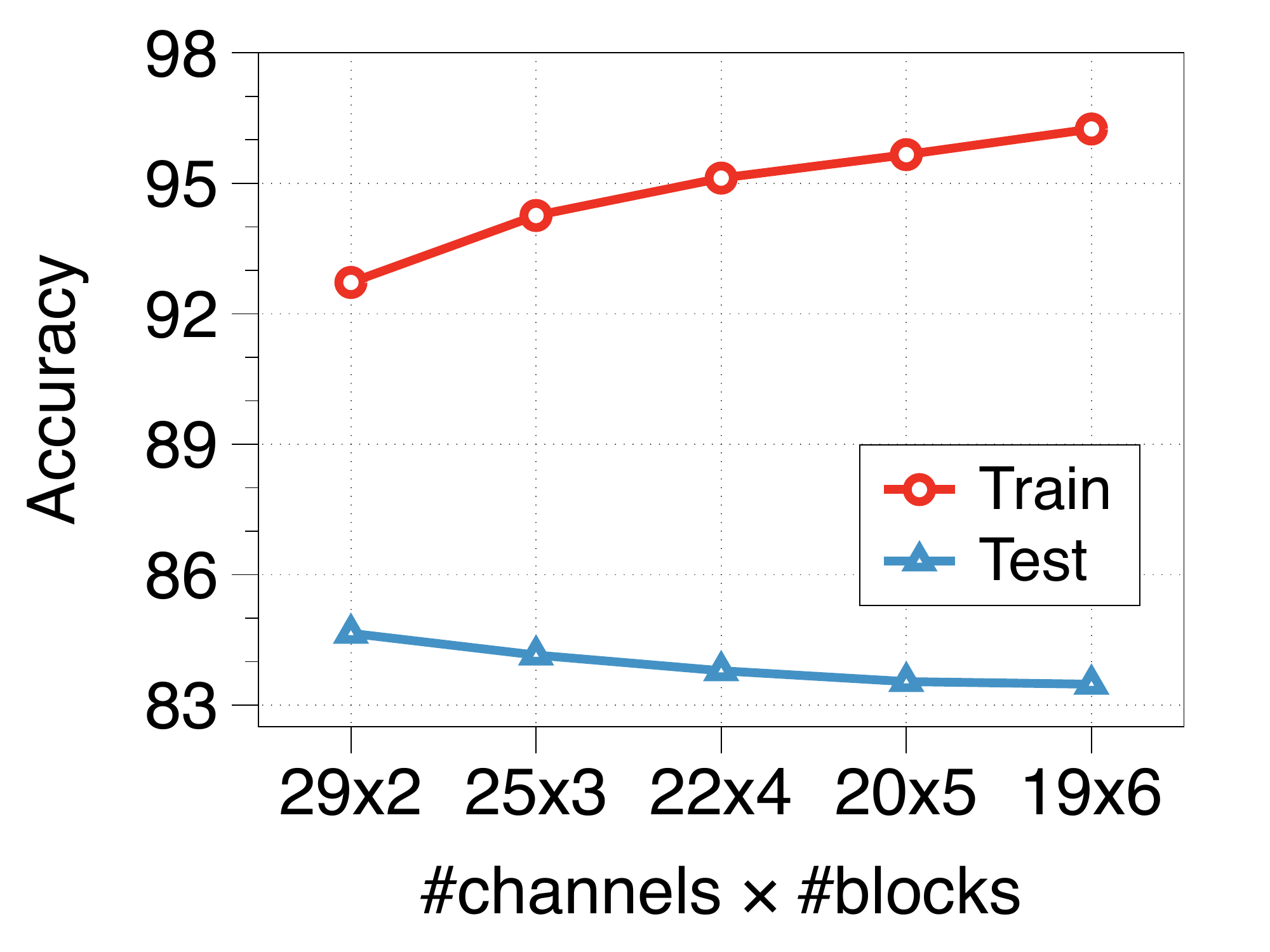}}
\subfigure[{SVHN}] {\includegraphics[width=2.2in]{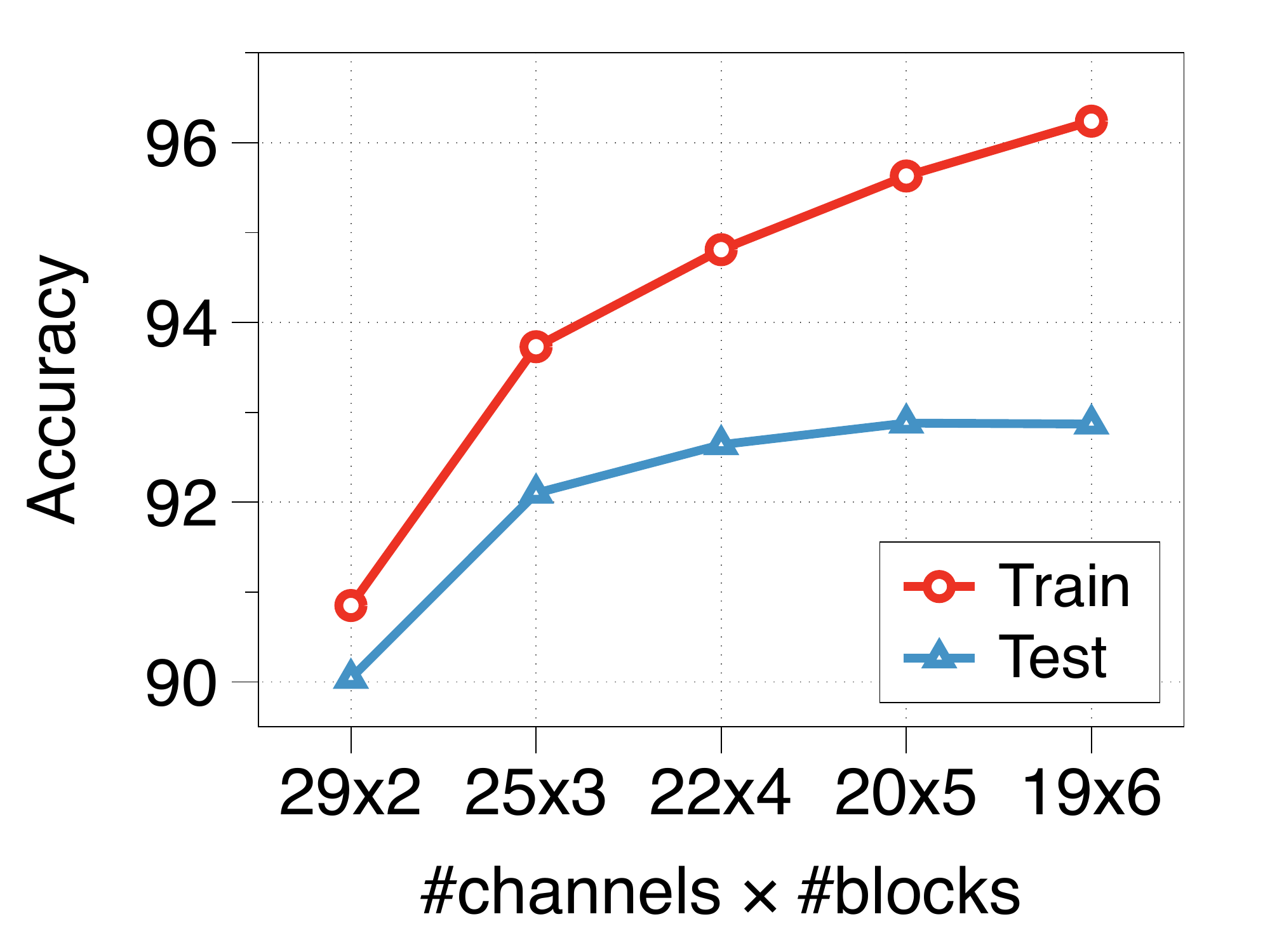}\label{fig:similar-svhn}}
\caption{Depth-width trade-off under a similar number of parameters.}
\label{fig:similar}
\end{figure}
\subsection{Depth-width trade-off in memorization}\label{sec:lwtradeoff}
We verify the memorization power of different network architectures having similar numbers of parameters.
Figure \ref{fig:similar} illustrates training and test accuracy of five different architectures with approximately $50000$ parameters for classifying the CIFAR-10 dataset \citep{krizhevsky09} and the SVHN dataset \citep{netzer11}.
One can observe that as the network architecture becomes deeper and narrower, the training accuracy increases.
Namely, deep and narrow networks memorize better than shallow and wide networks under similars number of parameters.
This observation agrees with Theorem \ref{thm:lw}, which states that increasing depth reduces the required number of parameters for memorizing the same number of pairs. 

However, more memorization power does not always imply better generalization (i.e., test accuracy).
In Figure~\ref{fig:similar}(b),
as the depth increases, the test accuracy also increases for the SVHN dataset. In contrast, the test accuracy decreases for the CIFAR-10 dataset as the depth increases in Figure~\ref{fig:similar}(a). 
In other words, overfitting occurs for the CIFAR-10 dataset while classifying the SVHN data receive benefits from more memorization power.
Note that a similar observation has also been made in the recent double descent phenomenon \citep{belkin19,nakkiran20} that more memorization power can both hurt/improve the generalization. 

\subsection{Effect of width and depth}
In this section, we observe the effect of depth and width by varying both.
Figure \ref{fig:heatmap-cifar} reports the training and test accuracy for the CIFAR-10 dataset by varying the number for channels from $5$ to $30$ and the number of residual blocks from $5$ to $50$. We present the experimental results for the SVHN dataset in Appendix \ref{sec:svhnapdx} under the same setup.
First, we observe that networks of $15$ channels with feature map size $32\times32$ successfully memorize (i.e., training accuracy over $99\%$).
This size is much narrower than modern network architectures, e.g., ResNet-18 has $64$ channels at the first hidden layer \citep{he16}.
On the other hand, too narrow networks (e.g., $5$-channels) fail to memorize.
This result does not contradict Theorem~\ref{thm:3} as the test of memorization in experiments/practice involves the stochastic gradient descent.
We note that similar observations are made for the SVHN dataset.

Furthermore, 
once the network memorize, we observe that increasing width is more effective than increasing depth for improving {test accuracy}.
These results indicate that width is not very critical for the memorization power; however, it can be effective for generalization.
Note that similar connections between generalization and the width/depth have also been made \citep{zhang17,golubeva21}.
\begin{figure}[t]
\centering
\begin{minipage}[b]{2.2in}
    \includegraphics[width=2.2in]{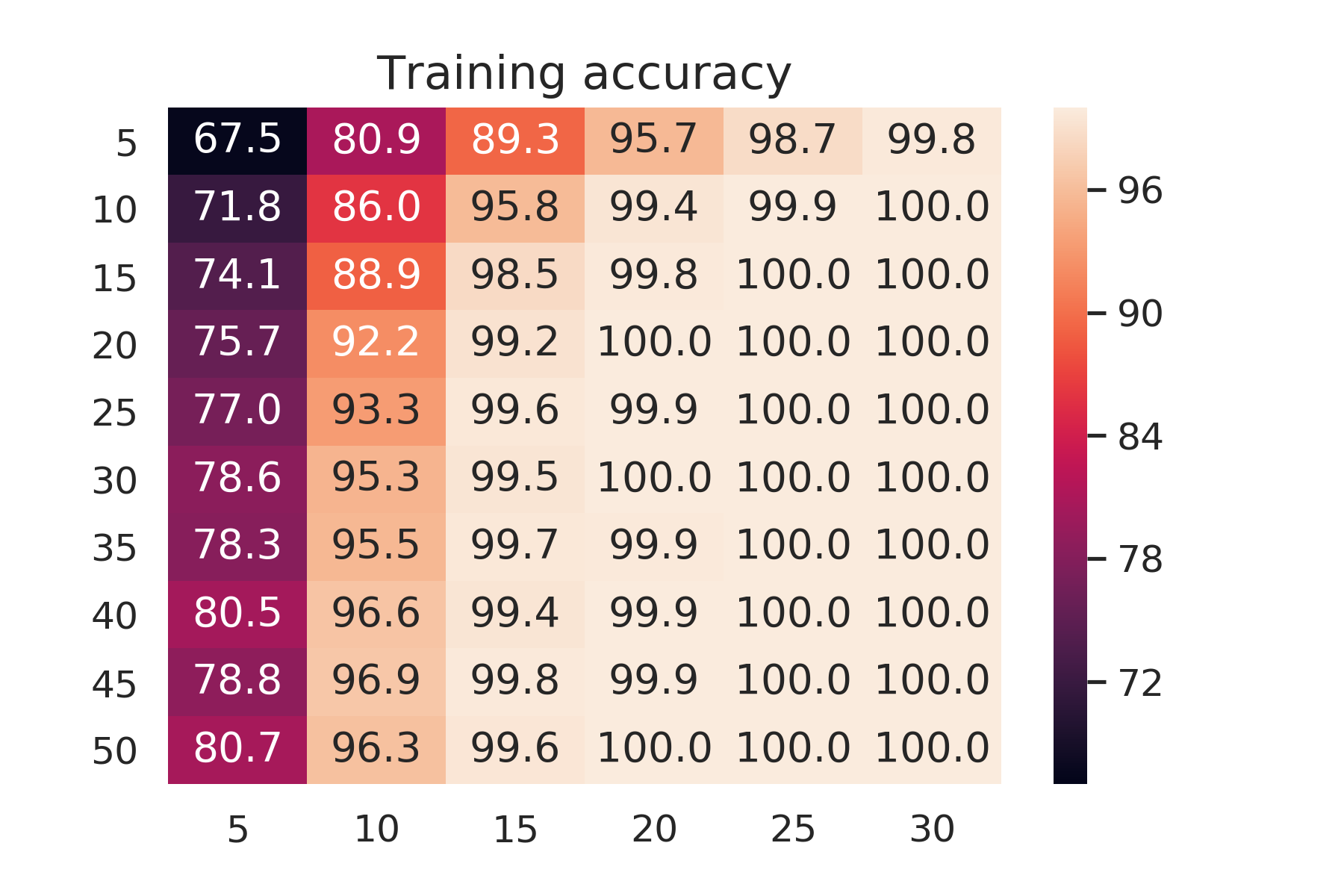}
\end{minipage}
\begin{minipage}[b]{2.2in}
    \includegraphics[width=2.2in]{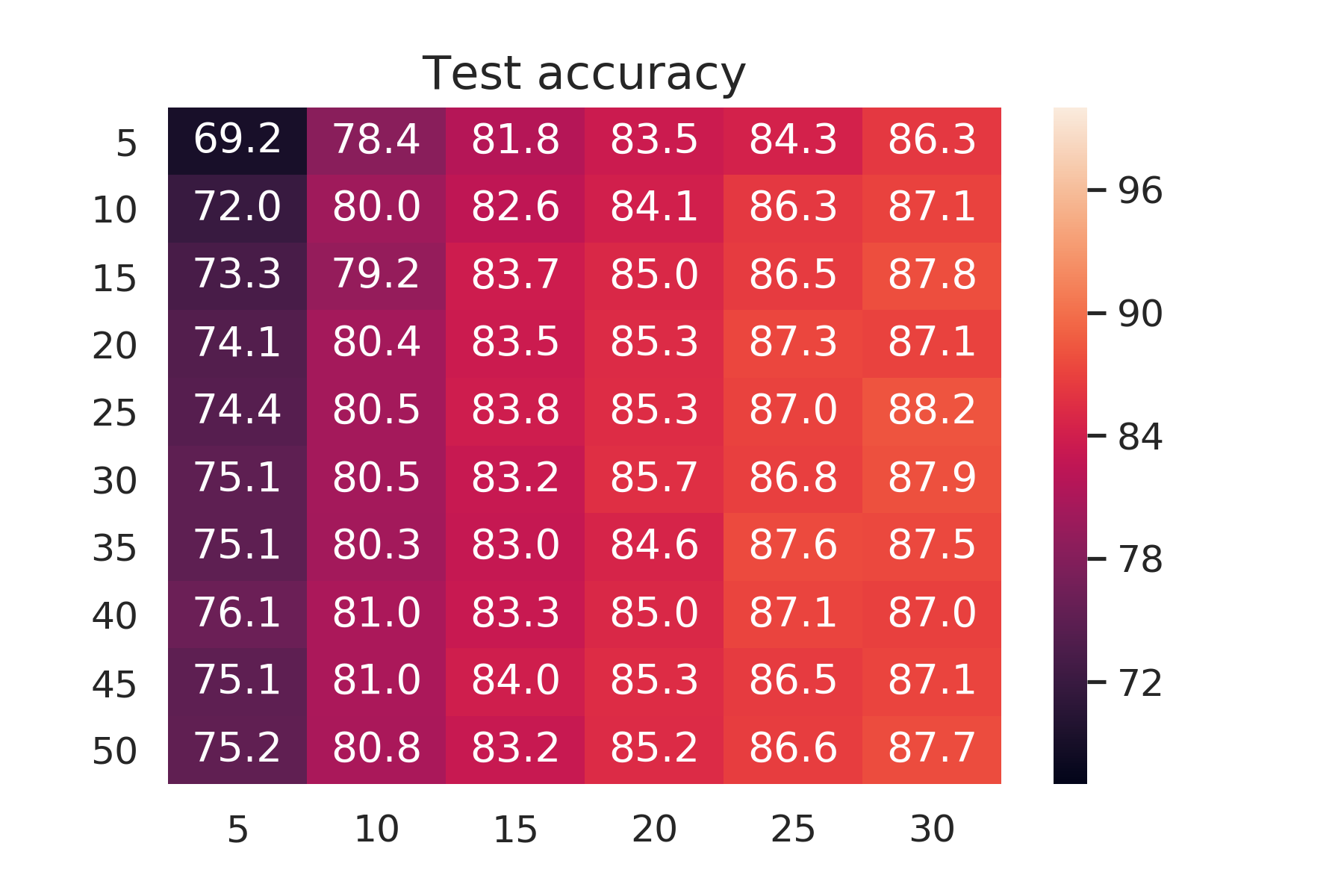}
\end{minipage}
\caption{Training/test accuracy of networks for training the CIFAR-10 dataset by varying the number of channels ($x$-axis) and the number of residual blocks ($y$-axis).}
\label{fig:heatmap-cifar}
\end{figure}

\section{Conclusion}\label{sec:conclusion}
In this paper, we prove that $\Theta(N^{2/3})$ parameters suffice for memorizing arbitrary $N$ pairs under the mild $\Delta$-separateness condition.
Our result provides significantly improved results, compared to the prior results showing the sufficiency of $\Theta(N)$ parameters with/without conditions on pairs.
Theorem~\ref{thm:lw} shows that deeper networks have more memorization power. This result coincides with the recent study on the benefits of depth for function approximation.
On the other hand, Theorem~\ref{thm:3} shows that network width is not important for the memorization power. We also provide sufficient conditions for identifying the memorization power of networks. Finally, we empirically confirm our theoretical results.

\acks{This work was supported by Institute of Information \& Communications Technology Planning \& Evaluation (IITP) grant funded by the Korea government (MSIT) 
(No.2019-0-00075, Artificial Intelligence Graduate School Program (KAIST) and No.2017-0-01779, A machine learning and statistical inference framework for explainable artificial intelligence). CY acknowledges Korea Foundation for Advanced Studies, NSF CAREER grant 1846088, and ONR grant
N00014-20-1-2394 for financial support.}

\clearpage
\bibliography{reference}

\appendix
\newpage
\onecolumn
{
\section{Experimental setup}\label{sec:expsetup}
In this section, we described the details on residual network architectures and hyper-parameter setups used for our experiments.

We use the residual networks of the following structure.
First, a convolutional layer and $\relu$ maps a $3$-channel input image to a $c$-channel feature map. Here, the size of the feature map is identical to the size of input images.
Then, we apply $b$ residual blocks where each residual block maps $x\mapsto\relu(\textsc{Conv}\circ\relu\circ\textsc{Conv}(x)+x)$ while preserving the number of channels and the size of feature map.
Finally, we apply an average pooling layer and a fully-connected layer.
We train the model for $5\times10^5$ iterations with batch size $64$ by the stochastic gradient descent.
We use the initial learning rate $0.1$, weight decay $10^{-4}$, and the learning rate decay at the  $1.5\times10^5$-th iteration and the $3.5\times10^5$-th iteration by a multiplicative factor $0.1$.
All presented results are averaged over three independent trials.
}

\newpage
\section{Training and test accuracy for SVHN dataset}\label{sec:svhnapdx}
\begin{figure}[h]
\centering
\subfigure[]{\includegraphics[width=2.5in]{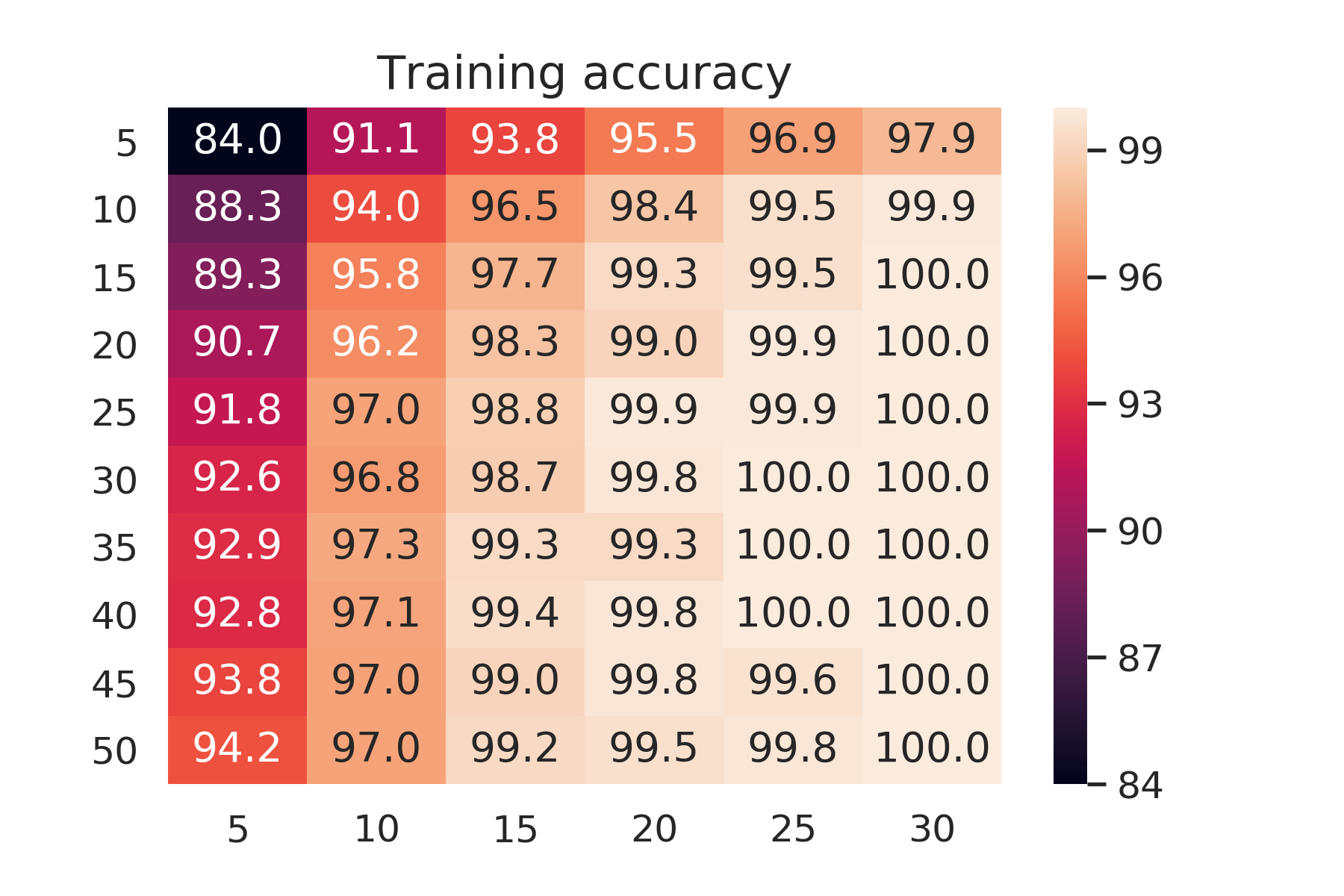}}
\subfigure[] {\includegraphics[width=2.5in]{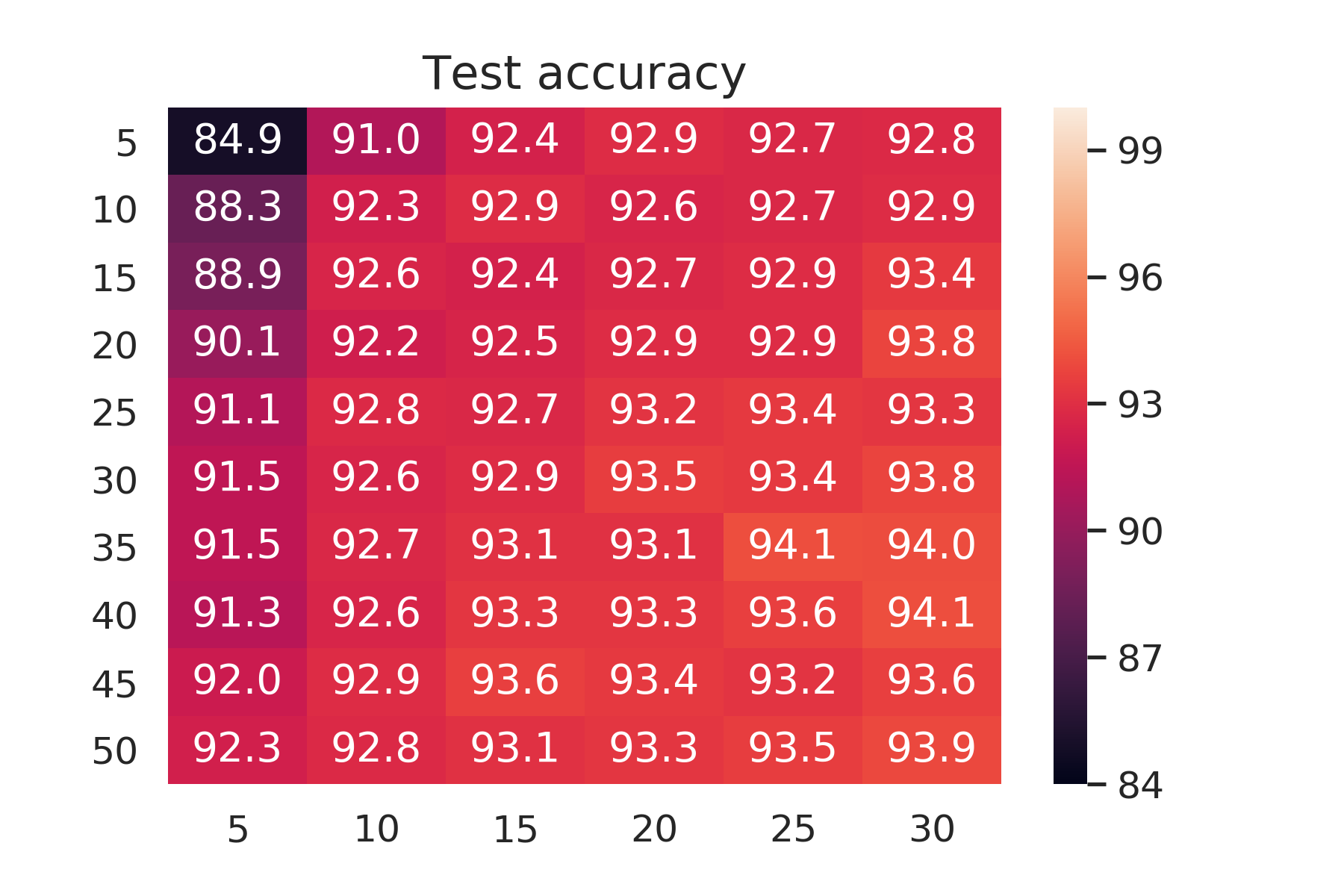}}
\caption{Training/test accuracy of networks for training the SVHN dataset by varying the number of channels ($x$-axis) and the number of residual blocks ($y$-axis).}
\label{fig:heatmap}
\end{figure}

\newpage
\section{$\Delta$-separateness of Gaussian random vectors}\label{sec:gaussian}
While we mentioned in Section \ref{sec:contribution} that digital nature of data enables the $\Delta$-separateness of inputs with small $\Delta$, random inputs are also $\Delta$-separated with small $\Delta$ with high probability.
In particular, we prove the following lemma.
\begin{lemma}\label{lem:gaussian}
For any $d_x\in\mathbb N$, consider a set of $N$ vectors $\mathcal X=\{x_1,\dots,x_N\}\subset\mathbb R^{d_x}$ where each entry of $x_i$ is drawn from the i.i.d.\ standard normal distribution. Then, for any $\delta>0$, $\mathcal X$ is $\Big((N/\sqrt{\delta})^{2/d_x}\sqrt{3e+\frac{5e}{d_x}\ln( {N}/{\sqrt\delta})}\Big)$-separated with probability at least $1-\delta$.
\end{lemma}
Lemma \ref{lem:gaussian} implies that Theorem \ref{thm:lw} and Theorem \ref{thm:3} can be successfully applied to random Gaussian input vectors
as the separateness condition in Lemma \ref{lem:gaussian} is much weaker than our $2^{O(N^{2/3})}$-separateness condition for memorization with $o(N)$ parameters.

\begin{proof}
First, notice that for $i\ne j$, we have
\begin{align}
    &\mathbb P\left(\sqrt{\frac2e}\cdot\left(\frac{\sqrt\delta}{N}\right)^{2/d_x}\le\frac1{\sqrt{d_x}}\|x_i-x_j\|_2\le\sqrt{6+\frac{10}{d_x}\ln \frac{N}{\sqrt\delta}}\right)\notag\\
    &=\mathbb P\left(\frac2e\cdot\left(\frac{\sqrt{\delta}}{N}\right)^{4/d_x}\le \frac1{d_x}\|x_i-x_j\|_2^2\le6+\frac{10}{d_x}\ln \frac{N}{\sqrt\delta}\right)\notag\\
    &=\mathbb P\left(\frac2e\cdot\left(\frac{\sqrt{\delta}}{N}\right)^{4/d_x}\le \frac2{d_x}X\le6+\frac{10}{d_x}\ln \frac{N}{\sqrt\delta}\right)\notag\\
    &=1-\mathbb P\left(\frac2e\cdot\left(\frac{\sqrt{\delta}}{N}\right)^{4/d_x}> \frac2{d_x}X\right)-\mathbb P\left(\frac2{d_x}X>6+\frac{10}{d_x}\ln\frac{N}{\sqrt\delta}\right)\notag\\
    &=1-\mathbb P\left(\frac1e\cdot\left(\frac{\sqrt{\delta}}{N}\right)^{4/d_x}\cdot d_x>X\right)-\mathbb P\left(X>\Big(3+\frac{5}{d_x}\ln \frac{N}{\sqrt\delta}\Big)\cdot d_x\right)\notag\\
    &\ge1-\left(\frac1e\cdot\left(\frac{\sqrt{\delta}}{N}\right)^{4/d_x}e^{1-\frac1e\cdot(\frac{\sqrt{\delta}}{N})^{4/d_x}}\right)^{d_x/2}-\left(\Big(3+\frac{5}{d_x}\ln\frac{N}{\sqrt\delta}\Big)e^{-2}\left(\frac{\sqrt{\delta}}{N}\right)^{5/d_x}\right)^{d_x/2}\notag\\
    &=1-\frac\delta{N^2}-\frac\delta{N^2}\cdot\left(\Big(3+\frac{5}{d_x}\ln \frac{N}{\sqrt\delta}\Big)e^{-2}\Big(\frac{\sqrt{\delta}}{N}\Big)^{1/d_x}\right)^{d_x/2}\notag\\
    &\ge 1-\frac\delta{N^2}-\frac\delta{N^2}\cdot\left(\frac3{e^2N^{1/d_x}}+\frac{5}{e^2}\cdot\frac{\ln \big(\frac{N}{\sqrt\delta}\big)^{1/d_x}}{\big(\frac{N}{\sqrt\delta}\big)^{1/d_x}}\right)^{d_x/2}\notag\\
    &\ge 1-\frac\delta{N^2}-\frac\delta{N^2}\cdot\left(\frac3{e^2}+\frac{5}{e^3}\right)^{d_x/2}\notag\\
    &\ge1-\frac\delta{N^2}-\frac\delta{N^2}=1-\frac{2\delta}{N^2}\label{eq:gaussian}
\end{align}
where $X$ denotes a chi-square random variable with $d_x$ degrees of freedom.
For the first inequality in \eqref{eq:gaussian}, we use the inequalities
\begin{align*}
    &\mathbb P(X< z\cdot d_x)\le\inf_{t>0}\frac{\mathbb{E}[e^{-tX}]}{e^{-t(z\cdot d_x)}}=\inf_{t>0}\frac{(1+2t)^{-d_x/2}}{e^{-t(z\cdot d_x)}}=(ze^{1-z})^{d_x/2}~\text{for}~0<z<1\\
    &\mathbb P(X> z\cdot d_x)\le\inf_{t>0}\frac{\mathbb{E}[e^{tX}]}{e^{t(z\cdot d_x)}}=\inf_{0<t<1/2}\frac{(1-2t)^{-d_x/2}}{e^{t(z\cdot d_x)}}=(ze^{1-z})^{d_x/2}~\text{for}~z>1
\end{align*}
which directly follow from the Chernoff bound for the chi-square distribution.
For the third inequality in \eqref{eq:gaussian}, we use the fact that $\max_{x>0}(\ln x)/{x}=1/e$ and $N\ge1$. For the last inequality in \eqref{eq:gaussian}, we use $\frac{3}{e^2}+\frac5{e^3}<1$.

Then, $\mathcal X$ is $\Big((N/\sqrt{\delta})^{2/d_x}\sqrt{3e+\frac{5e}{d_x}\ln( {N}/{\sqrt\delta})}\Big)$-separated with probability at least $1-\delta$ since the following bound holds:
\begin{align*}
&\mathbb P\left(\sqrt{\frac2e}\delta^{1/d_x}N^{-2/d_x}\le\frac1{\sqrt{d_x}}\|x_i-x_j\|_2\le\sqrt{6+\frac{10}{d_x}\ln \frac{N}{\sqrt\delta}},~\forall i\ne j\right)\\
&=1-\mathbb{P}\left(\exists i\ne j~\text{s.t.}~\|x_i-x_j\|_2> \sqrt{6+\frac{10}{d_x}\ln \frac{N}{\sqrt\delta}}~\text{or}~\|x_i-x_j\|_2< \sqrt{\frac2e}\delta^{1/d_x}N^{-2/d_x}\right)\\
&\ge1-\sum_{i\ne j}\mathbb{P}\left(\|x_i-x_j\|_2> \sqrt{6+\frac{10}{d_x}\ln \frac{N}{\sqrt\delta}}~\text{or}~\|x_i-x_j\|_2< \sqrt{\frac2e}\delta^{1/d_x}N^{-2/d_x}\right)\\
&=1-\sum_{i\ne j}\left(1-\mathbb P\left(\sqrt{\frac2e}\delta^{1/d_x}N^{-2/d_x}\le\frac1{\sqrt{d_x}}\|x_i-x_j\|_2\le\sqrt{6+\frac{10}{d_x}\ln \frac{N}{\sqrt\delta}}\right)\right)\\
&\ge1-\frac{N(N-1)}{2}\times \frac{2\delta}{N^2}\\
&\ge1-\delta
\end{align*}
where the first inequality follows from the union bound and the second inequality follows from \eqref{eq:gaussian}.
This completes the proof of Lemma \ref{lem:gaussian}.
\end{proof}

\newpage

\section{Proof of Lemmas for Theorem \ref{thm:lw}}\label{sec:pfthm:lw}

\subsection{Tools for proving Lemma \ref{lem:transform}}\label{sec:tools}
We present the following claims for the proving lemmas.
In particular, Claim \ref{claim:identity} and Claim \ref{claim:indicator} guide how to approximate $\id$ and $\step$ by a single sigmoidal neuron, respectively.
\begin{claim}[{\citet[Lemma~4.1]{kidger19}}]\label{claim:identity}
For any sigmoidal activation function $\sigma$, for any bounded interval $\mathcal I\subset \mathbb R$ for any $\varepsilon>0$, there exist $a,b,c,d\in\mathbb R$ such that $|a\cdot\sigma(c\cdot x+d)+b-x|<\varepsilon$ for all $x\in\mathcal I$.
\end{claim}
\begin{claim}\label{claim:indicator}
For any sigmoidal activation function $\sigma$, for any $\varepsilon,\delta>0$, there exist $a,b,c\in\mathbb R$ such that $|a\cdot\sigma(c\cdot x)+b-\mathbf{1}[x\ge0]|<\varepsilon$ for all $x\notin\mathcal[-\delta,\delta]$.
\end{claim}
\begin{proof}
We assume that $\alpha:=\lim_{x\rightarrow-\infty}\sigma(x)<\lim_{x\rightarrow\infty}\sigma(x)=:\beta$ where the case that $\beta<\alpha$ can be proved in a similar manner.
From the definition of $\alpha,\beta$, there exists $k>0$ such that $|\sigma(x)-\alpha|<(\beta-\alpha)\varepsilon$ if $x<-k$ and $|\sigma(x)-\beta|<(\beta-\alpha)\varepsilon$ if $x>k$.
Then, choosing $a=\frac1{\beta-\alpha}$, $b=-\frac{\alpha}{\beta-\alpha}$, and $c=\frac{k}{\delta}$ completes the proof of Claim \ref{claim:indicator}.
\end{proof}
\begin{claim}\label{claim:seqceil}
For any $a,x\in\mathbb R$ such that $a\ne0$, for any $b\in\mathbb N$, it holds that $\big\lceil\frac{x}{a\cdot b}\big\rceil=\big\lceil\frac{\lceil x/a\rceil}{b}\big\rceil$.
\end{claim}
\begin{proof}
It trivially holds that $\lceil\frac{x}{a\cdot b}\rceil\le\lceil\frac{\lceil x/a\rceil}{b}\rceil$.
Now, we show a contradiction if $\lceil\frac{x}{a\cdot b}\rceil<\lceil\frac{\lceil x/a\rceil}{b}\rceil$.
Suppose that $\lceil\frac{x}{a\cdot b}\rceil<\lceil\frac{\lceil x/a\rceil}{b}\rceil$.
Then, there exists an integer $m$ such that
\begin{align*}
\frac{x}{a\cdot b}\le m <\frac{\lceil x/a\rceil}{b}\qquad\text{and hence,}\qquad\frac{x}a\le\Big\lceil\frac{x}a\Big\rceil\le b\cdot m<\Big\lceil\frac{x}a\Big\rceil
\end{align*}
which leads to a contradiction.
This completes the proof of Claim \ref{claim:seqceil}.
\end{proof}

\subsection{Proof of Lemma \ref{lem:transform}}\label{sec:pflem:transform}
Without loss of generality, we first assume that for any network $h$ having $\step$ neurons, all inputs to  $\step$ neurons (i.e., $\mathbf{1}[x\ge0]$) is non-zero for all $x\in\mathcal X$ during the evaluation of $h(x)$.
This assumption can be easily satisfied by adding some small bias to the inputs of $\step$ neurons so that the output of $\step$ neurons does not change for all $x\in\mathcal X$. Note that such a bias always exists since $|\mathcal X|<\infty$. Furthermore, introducing this assumption does not affect to the values of $h(x)$ for all $x\in\mathcal X$

Now, we describe our construction of $g$.
Let $\delta>0$ be the minimum absolute value among all inputs to $\step$ neurons during the evaluation of $f(x)$ for all $x\in\mathcal X$.
Let $L$ be the number of hidden layers in $f$. 
Starting from $f$, we iteratively substitute the $\step$ and $\id$ hidden neurons into $\sigma$ hidden neurons, from the last hidden layer to the first hidden layer.
In particular, by using Claim \ref{claim:identity} and Claim \ref{claim:indicator}, we replace $\step$ and $\id$ neurons in the $\ell$-th hidden layer by $\sigma$ neurons approximating $\id$ and $\step$.

First, let $g_L$ be a network identical to $f$ except for its $L$-th hidden layer. In particular, the $L$-th hidden layer of $g_L$ consists of $\sigma$ neurons approximating $\step$ and $\id$  neurons in the $L$-th hidden layer of $f$.
Here, we accurately approximate $\step$ and $\id$ neurons by $\sigma$ using Claim \ref{claim:identity} and Claim \ref{claim:indicator} so that $|g_L(x)-f(x)|<\varepsilon/L$ for all $x\in\mathcal X$.
Note that such approximation always exists due to the existence of $\delta>0$.
Now, let $g_{L-1}$ be a network identical to $f_L$ except for its $(L-1)$-th hidden layer consisting of $\sigma$ neurons approximating $\step$ and $\id$ neurons in the $(L-1)$-th hidden layer of $g_L$.
Here, we also accurately approximate $\step$ and $\id$ neurons by $\sigma$ using Claim \ref{claim:identity} and Claim \ref{claim:indicator} so that $|g_L(x)-g_{L-1}(x)|<\varepsilon/L$ for all $x\in\mathcal X$.
If we repeat this procedure until replacing the first hidden layer, then $g:=g_1$ would be the desired network satisfying that $|f(x)-g(x)|<\varepsilon$ for all $x\in\mathcal X$.
This completes the proof of Lemma \ref{lem:transform}.

\subsection{Proof of Lemma \ref{lem:projection-rough}}\label{sec:pflem:projection-rough}
We let $d_x \ge 2$, as the lemma trivially holds when $d_x = 1$.
Now, consider a projection $x \mapsto u^\top x \in \mathbb{R}$ by some vector $u \in \mathbb{R}^{d_x}$. There exists some $u$ such that
\begin{align}
\frac{\max_{\{x,x^\prime\}\in\binom{\mathcal X}{2}}|u^\top(x-x^\prime)|}{\min_{\{x,x^\prime\}\in\binom{\mathcal X}{2}}|u^\top(x-x^\prime)|}<N^2\Delta\sqrt{\frac{\pi d_x}{8}}.\label{eq:projection}
\end{align}
holds, as we can see from the following lemma. The proof of Lemma~\ref{lem:projection} is presented in Appendix~\ref{sec:pflem:projection}.
\begin{lemma}\label{lem:projection}
For any $N,d_x\in\mathbb{N}$, for any $\mathcal X\in\binom{\mathbb{R}^{d_x}}{N}$, there exists a unit vector $u\in\mathbb{R}^{d_x}$ such that $\frac1{N^2}\sqrt{\frac{8}{\pi d_x}}\|x-x^\prime\|_2\le|u^\top(x-x^\prime)|\le\|x-x^\prime\|_2$ for all $x,x^\prime\in\mathcal X$.
\end{lemma}
Finally, we construct the desired map by 
$$x\mapsto\frac{u^\top x-\min\{u^\top x:x\in\mathcal X\}}{\min_{\{x,x^\prime\}\in\binom{\mathcal X}{2}}|u^\top(x-x^\prime)|}=:v^\top x+b.$$
Then, the following equalities/inequalities holds:
\begin{align*}
    &\min_{\{x,x^\prime\}\in\binom{\mathcal X}{2}}|v^\top(x-x^\prime)|=\frac{\min_{\{x,x^\prime\}\in\binom{\mathcal X}{2}}|u^\top (x-x^\prime)|}{\min_{\{x,x^\prime\}\in\binom{\mathcal X}{2}}|u^\top(x-x^\prime)|}=1,\\
    &\max_{\{x,x^\prime\}\in\binom{\mathcal X}{2}}|v^\top(x-x^\prime)|=\frac{\max_{\{x,x^\prime\}\in\binom{\mathcal X}{2}}|u^\top (x-x^\prime)|}{\min_{\{x,x^\prime\}\in\binom{\mathcal X}{2}}|u^\top(x-x^\prime)|}\le N^2\Delta\sqrt{\frac{\pi d_x}{8}},\\
    &\min_{x\in\mathcal X}v^\top x+b=0,\\
    &\max_{x\in\mathcal X}v^\top x+b=\max_{\{x,x^\prime\}\in\binom{\mathcal X}{2}}|v^\top (x-x^\prime)|\le N^2\Delta\sqrt{\frac{\pi d_x}{8}},
\end{align*}
i.e., $\lfloor\{v^\top x+b:x\in\mathcal X\}\rfloor\in\binom{[\lceil N^2\Delta\sqrt{{\pi d_x}/{8}}\rceil]}{N}$.
This completes the proof of Lemma \ref{lem:projection-rough}.

\subsection{Proof of Lemma \ref{lem:compression1-rough}}\label{sec:pflem:compression1-rough}
To prove Lemma \ref{lem:compression1-rough}, we introduce its generalization as follows. The proof of Lemma \ref{lem:compression1-free} is presented in Appendix \ref{sec:pflem:compression1-free}.

\begin{lemma}\label{lem:compression1-free}
For any $N, K, d\in\mathbb N$ and for any $\mathcal X$ such that $\lfloor\mathcal X\rfloor\in\binom{[K]}{N}$, there exists a $\step+\id$ network $f$ of $1$ hidden layer and width $d$ such that $\lfloor f(\mathcal X)\rfloor\in\binom{[T]}{N}$ where $T:=\max\big\{\big\lceil\frac{K}{\lfloor(d+1)/2\rfloor}\big\rceil,\lfloor\frac{N^2}{4}+1\rfloor\big\}$.
\end{lemma}
From Lemma \ref{lem:compression1-free}, one can observe that a $\step+\id$ network of $1$ hidden layer consisting of $3$ hidden neuron which can map any $\mathcal Z$ such that $\lfloor\mathcal Z\rfloor\in\binom{[K]}{N}$ to $\mathcal Z^\prime$ such that $\lfloor\mathcal Z^\prime\rfloor\in\binom{[T]}{N}$ where $T:=\max\big\{\big\lceil\frac{K}{\lfloor(d+1)/2\rfloor}\big\rceil,\lfloor\frac{N^2}{4}+1\rfloor\big\}$.
Combining Lemma \ref{lem:compression1-free} and Lemma \ref{lem:transform} completes the proof of Lemma \ref{lem:compression1-rough}.

\subsection{Proof of Lemma \ref{lem:compression2-rough}}\label{sec:pflem:compression2-rough}
To prove Lemma \ref{lem:compression2-rough}, we introduce its generalization as follows. The proof of Lemma \ref{lem:compression2-free} is presented in Appendix \ref{sec:pflem:compression2-free}.
\begin{lemma}\label{lem:compression2-free}
For any $N,K,L,d_1,\dots,d_L\in\mathbb N$ such that $d_\ell\ge3$ for all $\ell$, for any $\mathcal X$ such that $\lfloor\mathcal X\rfloor\in\binom{[K]}{N}$, there exists a $\step+\id$ network $f$ of $L$ hidden layers having $d_\ell$ neurons at the $\ell$-th hidden layer such that $\lfloor f(\mathcal X)\rfloor\in\binom{[T]}{N}$ where $T:=\min\Big\{K,\max\big\{N\times\lceil{N}/{C}\rceil,\lceil{K}/{2}\rceil\big\}\Big\}$ and $C:=\lfloor\frac12+\frac12\sum_{\ell=1}^L(d_\ell-2)\rfloor$.
\end{lemma}
We choose $C=\lceil4N^2/K\rceil$ in Lemma \ref{lem:compression2-free} so that
\begin{align*}
    N\times\lceil N/C\rceil\le N\times\lceil K/(4N)\rceil\le\max\{N,K/2\}.
\end{align*}
Then, a $\step+\id$ network $f$ of one hidden layer having $\Theta(N^2/K)$ hidden neurons can map any $\mathcal X$ such that $\lfloor\mathcal X\rfloor\in\binom{[K]}{N}$ to $f(\mathcal X)$ such that $\lfloor f(\mathcal X)\rfloor=\binom{[T]}{N}$ where $T:=\max\{N,\lceil{K}/2\rceil\}$. Combining this and Lemma \ref{lem:transform} completes the proof of Lemma \ref{lem:compression2-rough}.
\subsection{Proof of Lemma \ref{lem:memorizer-rough}}\label{sec:pflem:memorizer-rough}
To prove Lemma \ref{lem:memorizer-rough}, we introduce the following lemma. We note that Lemma \ref{lem:learning} follows the construction for proving VC-dimension lower bound of $\relu$ networks \citep{bartlett19}.
The proof of Lemma \ref{lem:learning} is presented in Appendix \ref{sec:pflem:learning}.
\begin{lemma}\label{lem:learning}
For any $A,B,D,K,R\in\mathbb N$ such that $AB\ge K$, there exists a $\step+\id$ network $f_\theta$ of $2\lceil\frac{BD}{R}\rceil+2$ hidden layers and $4A+\big((2R+5)2^R+2R^2+8R+7\big)\lceil\frac{BD}{R}\rceil-R2^R-R^2+3$ parameters satisfying the following property:
For any finite set $\mathcal X\subset[0,K)$, for any $y:[K]\rightarrow[2^D]$, there exists $\theta$ such that
$f_\theta(x)=y(\lfloor x\rfloor)$ for all $x\in\mathcal X$.
\end{lemma}
By choosing $K\leftarrow V$, $A\leftarrow \Theta(V^p)$, $B\leftarrow \Theta(V^{1-p})$, $R\leftarrow1+(p-0.5)\log V$, and $D\leftarrow\lceil\log C\rceil$ in Lemma \ref{lem:learning}, there exists a $\step+\id$ network of $O\big(\frac{V^{1-p}}{1+(p-0.5)\log V}\log C\big)$ hidden layers and $O(V^p+V^{1/2}\log C)$ parameters which can memorize any $\mathcal X\subset\mathbb R$ of size $V$ with $C$ classes satisfying $\lfloor\mathcal X\rfloor=[V]$. Combining this with Lemma \ref{lem:transform} completes the proof of Lemma \ref{lem:memorizer-rough}.

\newpage
\section{Proof of Theorem \ref{thm:criteria}}\label{sec:pfthm:criteria}
The proof of Theorem \ref{thm:criteria} have the same structure of the proof of Theorem \ref{thm:lw}. 
In particular, we construct a $\step+\id$ network and use Lemma \ref{lem:transform} as in the proofs of Theorem \ref{thm:lw} and Theorem \ref{thm:3}.

For the network construction, we divide the function of $f_\theta$ into four disjoint parts, as in Section \ref{sec:pfsketch:lw}.
The first part does not utilize a hidden layers but construct project input vectors into scalar values.
The second part corresponds to the first $L_1$ hidden layers decreases the upper bound on scalar values to $O(N^2)$.
The third part corresponds to the next $L_{K-1}-L_1$ hidden layers further decreases the upper bound to $o(N^2)$.
The last part corresponds to the rest hidden layers construct a network mapping hidden features to their labels. 

Now, we describe our construction in detail.
To begin with, let us denote a $\Delta$-separated set of inputs by $\mathcal X_1$. 
First, from Lemma \ref{lem:projection}, one can project $\mathcal X_1$ to $\mathcal X_2$ such that $\lfloor\mathcal X_2\rfloor\in\binom{[\lceil N^2\Delta\sqrt{\pi d_x/8}\rceil]}{N}$. Note that the projection step does not require to use hidden layers as it can be absorbed into the linear map before the next hidden layer.
Then, the first $L_1$ hidden layers can map $\mathcal X_2$ to $\mathcal X_3$ such that $\lfloor\mathcal X_3\rfloor\in\binom{[\lfloor N^2/4+1\rfloor]}{N}$ from Lemma \ref{lem:compression1-free} and the first condition:
\begin{align*}
    \prod_{\ell=1}^{L_1}\left\lfloor\frac{d_\ell+1}{2}\right\rfloor\ge\Delta\sqrt{2\pi d_x}
    \Rightarrow&\prod_{\ell=1}^{L_1}\left\lfloor\frac{d_\ell+1}{2}\right\rfloor\ge \frac{N^2\Delta\sqrt{\pi d_x/8}}{N^2/4}\\
    \Rightarrow&\frac{N^2}4\ge\frac{N^2\Delta\sqrt{\pi d_x/8}}{\prod_{\ell=1}^{L_1}\lfloor{(d_\ell+1)}/{2}\rfloor}\\
    \Rightarrow&\left\lfloor\frac{N^2}4+1\right\rfloor\ge\left\lceil\frac{N^2\Delta\sqrt{\pi d_x/8}}{\prod_{\ell=1}^{L_1}\lfloor{(d_\ell+1)}/{2}\rfloor}\right\rceil.
\end{align*}
Note that we also utilize Claim \ref{claim:seqceil} for the sequential application of Lemma \ref{lem:compression1-free}.
Consecutively, the next $L_{K-1}-L_1$ hidden layers can map $\mathcal X_3$ to $\mathcal X_4$ such that $\lfloor\mathcal X_4\rfloor\in\binom{[\lceil U/2^{K-2}\rceil]}{N}$ from Lemma \ref{lem:compression2-free} and the second condition:
\begin{align*}
\sum_{\ell=L_{i-1}+1}^{L_i}(d_\ell-2)\ge2^{i+3}
\Rightarrow&\frac12\sum_{\ell=L_{i-1}+1}^{L_i}(d_\ell-2)\ge2^{i+2}
\Rightarrow\Big\lfloor\frac12+\frac12\sum_{\ell=L_{i-1}+1}^{L_i}(d_\ell-2)\Big\rfloor\ge2^{i+2}\\
\Rightarrow&\frac{N}{\Big\lfloor\frac12+\frac12\sum_{\ell=L_{i-1}+1}^{L_i}(d_\ell-2)\Big\rfloor}\le\frac{N}{2^{i+2}}\\
\Rightarrow&\left\lceil\frac{N}{\Big\lfloor\frac12+\frac12\sum_{\ell=L_{i-1}+1}^{L_i}(d_\ell-2)\Big\rfloor}\right\rceil\le\frac{N}{2^{i+1}}\\
\Rightarrow&N\times\left\lceil\frac{N}{\Big\lfloor\frac12+\frac12\sum_{\ell=L_{i-1}+1}^{L_i}(d_\ell-2)\Big\rfloor}\right\rceil\le\frac{N^2/4}{2^{i-1}}\\
\Rightarrow&N\times\left\lceil\frac{N}{\Big\lfloor\frac12+\frac12\sum_{\ell=L_{i-1}+1}^{L_i}(d_\ell-2)\Big\rfloor}\right\rceil\le\left\lceil\frac{\lfloor N^2/4+1\rfloor}{2^{i-1}}\right\rceil
\end{align*}
where we use the inequality $\lceil a\rceil\le 2a$ for $a\ge1/2$ and the assumption $K\le\log N$, i.e., $N/2^{i+1}\ge1$ for $i\le\log N-1$.
Here, we also utilize Claim \ref{claim:seqceil} for the sequential application of Lemma \ref{lem:compression2-free}.

We reinterpret the Third condition as follows.
\begin{align*}
    &2^{K}\cdot\left(\sum_{\ell=L_{K-1}+1}^{L_K}(d_\ell-2)\right)\cdot\left\lfloor\frac{L-L_K}{2\lceil\log C\rceil+1}\right\rfloor\ge N^2+4\\
    \Rightarrow&\left(\sum_{\ell=L_{K-1}+1}^{L_K}(d_\ell-2)\right)\cdot\left\lfloor\frac{L-L_K}{2\lceil\log C\rceil+1}\right\rfloor\ge \frac{N^2/4+1}{2^{K-2}}\\
    \Rightarrow&\left(\sum_{\ell=L_{K-1}+1}^{L_K}(d_\ell-2)\right)\cdot\left\lfloor\frac{L-L_K}{2\lceil\log C\rceil+1}\right\rfloor\ge \left\lceil\frac{\lfloor N^2/4+1\rfloor}{2^{K-2}}\right\rceil
\end{align*}
Finally, from the following lemma and the above inequality, the rest hidden layers can map $\mathcal X_4$ to their corresponding labels (by choosing $L^\prime:=L_3$). The proof of Lemma \ref{lem:learning-free2} is presented in Appendix \ref{sec:pflem:learning-free2}. 

\begin{lemma}\label{lem:learning-free2}
For $D,K,L,d_1,\dots,d_L\in\mathbb N$ such that $d_\ell\ge3$ for all $\ell$, 
suppose that there exist $0<L^\prime<L$ satisfying that
\begin{align*}
    &\left(\sum_{\ell=1}^{L^\prime}(d_\ell-2)\right)\cdot\left\lfloor\frac{L-L^\prime}{2D+1}\right\rfloor\ge K.
\end{align*}
Then, there exists a $\step+\id$ network $f_\theta$ of $L$ hidden layers having $d_\ell$ hidden neurons at the $\ell$-th hidden layer such that for any finite $\mathcal X\subset[0,K)$, for any $y:[K]\rightarrow[2^D]$, there exists $\theta$ satisfying $f_\theta(x)=y(\lfloor x\rfloor)$ for all $x\in\mathcal X$.
\end{lemma}
\newpage

\section{Proofs of technical lemmas}\label{sec:mempftechlems}
\subsection{Proof of Lemma \ref{lem:projection}}\label{sec:pflem:projection}
We focus on proving the lower bound; the upper bound holds as $\|u\|_2 = 1$. We also let $N \ge 2$ and $d_x\ge2$, as the lemma trivially holds for a singleton $\mathcal{X}$ or $d_x=1$.
We prove the lemma by showing that for any vector $v \in \mathbb{R}^{d_x}$, a \textit{random} unit vector $u \in \mathbb{R}^{d_x}$ uniformly randomly drawn from the hypersphere $\mathbb{S}^{d_x-1}$ satisfies
\begin{align}
    \mathbb{P}\left(|u^\top v| < \frac{\|v\|_2}{N^2}\sqrt{\frac{8}{\pi d_x}}\right)
    &<\frac{2}{N^2}.\label{eq:projectionprbound}
\end{align}
Once we show \eqref{eq:projectionprbound}, the \textit{existence} of a unit vector $u$ satisfying the lower bound of Lemma~\ref{lem:projection} follows. To see this, define $\mathcal{V} := \{x - x'~:~\{x,x'\} \in \binom{\mathcal{X}}{2}, x \le x'\}$ for some total order $\le$ on $\mathcal{X}$. Then, the union bound implies
\begin{align*}
\mathbb{P}\left(\bigcup_{v\in\mathcal V}\left\{|u^\top v|<\frac{\|v\|_2}{N^2}\sqrt{\frac{8}{\pi d_x}}\right\}\right)\le\sum_{v\in\mathcal V}\mathbb{P}\left(|u^\top v|<\frac{\|v\|_2}{N^2}\sqrt{\frac{8}{\pi d_x}}\right)<\frac{N(N-1)}{2}\times\frac{2}{N^2}<1,
\end{align*}
and thus there exists at least one unit vector $u$ that satisfies the lower bound.

To show \eqref{eq:projectionprbound}, we begin by noting that 
\begin{align*}
    \mathbb{P}\left(|u^\top v|<\frac{\|v\|_2}{N^2}\sqrt{\frac{8}{\pi d_x}}\right)&=\mathbb{P}\left(|u_1|<\frac{1}{N^2}\sqrt{\frac{8}{\pi d_x}}\right)
\end{align*}
holds for any $v \in \mathbb{R}^{d_x}$ by the symmetry of the uniform distribution. Now we proceed as
\begin{align*}
    \mathbb{P}\left(|u_1|<\frac{1}{N^2}\sqrt{\frac{8}{\pi d_x}}\right)&=2\times\mathbb{P}\left(0<u_1<\frac{1}{N^2}\sqrt{\frac{8}{\pi d_x}}\right)\\
    &=\frac2{\mathtt{Area}(\mathbb{S}^{d_x-1})}\times\int_{\arccos\left(\frac{1}{N^2}\sqrt{\frac{8}{\pi d_x}}\right)}^{\frac\pi2}\mathtt{Area}(\mathbb{S}^{d_x-2})\cdot(\sin\phi)^{d_x-2}\textup{ d}\phi\\
    &=2\times\frac{\mathtt{Area}(\mathbb{S}^{d_x-2})}{\mathtt{Area}(\mathbb{S}^{d_x-1})}\times\int_{\arccos\left(\frac{1}{N^2}\sqrt{\frac{8}{\pi d_x}}\right)}^{\frac\pi2}(\sin\phi)^{d_x-2} \textup{ d}\phi\\
    &= \frac{2}{\sqrt{\pi}}\times\frac{(d_x-1)\Gamma(\frac{d_x}{2}+1)}{d_x\Gamma(\frac{d_x}{2}+\frac{1}{2})}\times\int_{\arccos\left(\frac{1}{N^2}\sqrt{\frac{8}{\pi d_x}}\right)}^{\frac\pi2}(\sin\phi)^{d_x-2} \textup{ d}\phi\\
    &< \sqrt{\frac{2}{\pi}}\times\frac{(d_x-1)\sqrt{d_x+2}}{d_x}\times\int_{\arccos\left(\frac{1}{N^2}\sqrt{\frac{8}{\pi d_x}}\right)}^{\frac\pi2} 1 \textup{ d}\phi\\
    &\le \sqrt{\frac{2d_x}{\pi}}\times\int_{\arccos\left(\frac{1}{N^2}\sqrt{\frac{8}{\pi d_x}}\right)}^{\frac\pi2}1 \textup{ d}\phi\\
    &=\sqrt{\frac{2d_x}{\pi}}\left(\frac\pi2-\arccos\left(\frac{1}{N^2}\sqrt{\frac{8}{\pi d_x}}\right)\right)\\
    &=\sqrt{\frac{2d_x}{\pi}}\arcsin\left(\frac{1}{N^2}\sqrt{\frac{8}{\pi d_x}}\right)\\
    &\le\sqrt{\frac{2d_x}{\pi}}\times\frac{\pi}{2}\times \frac{1}{N^2}\sqrt{\frac{8}{\pi d_x}}=\frac2{N^2}
\end{align*}
where $\mathtt{Area}(\cdot)$ denotes the surface area of the object, i.e., $\mathtt{Area}(\mathbb S^{d_x-1})=d_x\pi^{\frac{d_x}2}/\Gamma(\frac{d_x}2+1)$. Here, the first inequality follows from the Gautschi's inequality (see Lemma~\ref{lem:gautschi}) and $\sin \pi \le 1$. The second inequality follows from the fact that $\frac{(t-1)\sqrt{t+2}}{t} \le \sqrt{t}$ for $t \ge 1$. The third inequality follows from the fact that $\phi \le \frac{\pi}{2}\sin \phi$ for any $0 \le \phi \le \frac{\pi}{2}$. This completes the proof of Lemma~\ref{lem:projection}.

\begin{lemma}[Gautschi's inequality \citep{gautschi59}]\label{lem:gautschi}
For any $x>0, s\in(0,1)$,
\begin{align*}
    x^{1-s}<\frac{\Gamma(x+1)}{\Gamma(x+s)}<(x+1)^{1-s}.
\end{align*}
\end{lemma}

\newpage
\subsection{Proof of Lemma \ref{lem:compression1-free}}\label{sec:pflem:compression1-free}
In this proof, we assume that $K>N^2/4$ and $T=\big\lceil\frac{K}{\lfloor(d+1)/2\rfloor}\big\rceil$ since implementing the identity function is sufficient otherwise.
To begin with, we first define a network $f_b(x):[0,K)\rightarrow[0,T)$ for $b=(b_i)_{i=1}^{\lfloor(d-1)/2\rfloor}\in[T]^{\lfloor(d-1)/2\rfloor}$ as
\begin{align}
    &f_b(x):=\begin{cases}x&\text{if}~x<T\\
    (x+b_1)\bmod T&\text{if}~T\le x<2T\\
    (x+b_2)\bmod T&\text{if}~2T\le x<3T\\
    &\vdots\\
    (x+b_{\lfloor(d-1)/2\rfloor})\bmod T &\text{if}~\lfloor\tfrac{d-1}{2}\rfloor T\le x
    \end{cases}\notag\\
    &=x-T\times\mathbf{1}[x\ge T]+\sum_{i=1}^{\lfloor(d-1)/2\rfloor}\left(\bigg(b_i-\sum_{j=1}^{i-1}b_j\bigg)\times\mathbf{1}[x\ge iT]-T\times\mathbf{1}[x+b_i\ge (i+1)T]\right).\label{eq:compression1-free}
\end{align}
One can easily observe that $f_b$ can be implemented by a  $\step+\id$ network of $1$ hidden layer and width $d$ as $T\times\mathbf{1}[x\ge T]$ in \eqref{eq:compression1-free} can be absorbed into $\big(b_i-\sum_{j=1}^{i-1}b_j\big)\times\mathbf{1}[x\ge iT]$ in \eqref{eq:compression1-free} for $i=1$.

Now, we show that if $T>{N^2}/{4}$, then there exist $b\in[T]^{\lfloor(d-1)/2\rfloor}$ such that $\big|\lfloor f_{b}(\mathcal X)\rfloor \big|=N$ to complete the proof. 
Our proof utilizes the mathematical induction on $i$: If there exist $b_1,\dots,b_{i-1}\in[T]$ such that 
\begin{align}
\lfloor f_b(\{x\in\mathcal X:x< iT\})\rfloor=\binom{[T]}{\big|\lfloor\{x\in\mathcal X:x< iT\}\rfloor\big|},\label{eq:compression1-free-ih}
\end{align}
then there exists $b_i\in[T]$ such that \begin{align}
\lfloor f_b(\{x\in\mathcal X:x< (i+1)T\})\rfloor=\binom{[T]}{\big|\lfloor\{x\in\mathcal X:x< (i+1)T\}\rfloor\big|}.\label{eq:compression1-free-ir}
\end{align}
Here, one can observe that the statement trivially holds for the base case, i.e., for $\{x\in\mathcal X:x<T\}$.
Now, using the induction hypothesis, suppose that there exist $b_1,\dots,b_{i-1}\in[T]$ satisfying \eqref{eq:compression1-free-ih}. Now, we prove that there exists $b_i\in[T]$ such that $\mathcal S_{b_i}:=\lfloor f_b(\{x\in\mathcal X:iT\le x<(i+1)T\})\rfloor$ does not intersect with $\mathcal T:=\lfloor f_b(\{x\in\mathcal X:x< iT\})\rfloor$, i.e, \eqref{eq:compression1-free-ir} holds.
Consider the following inequality:
\begin{align*}
    \sum_{b_i\in [T]}|\mathcal S_{b_i}\cap\mathcal T|
    &=|\mathcal S_{b_i}|\times|\mathcal T|\le\frac{N^2}{4}
\end{align*}
where the equality follows from the fact that for each $x\in\{x\in\mathcal X:iT\le x<(i+1)T\}$, there exists exactly $|\mathcal T|$ values of $b_i$ so that $\lfloor x\bmod T\rfloor\in\mathcal T$.
However, since the number of possible choices of $b_i$ is $T$, if $T>{N^2}/{4}$, then there exists $b_i\in[T]$ such that $\mathcal S_{b_i}\cap\mathcal T=\emptyset$, i.e., \eqref{eq:compression1-free-ir} holds.
This completes the proof of Lemma \ref{lem:compression1-free}.

\newpage
\subsection{Proof of Lemma \ref{lem:compression2-free}}\label{sec:pflem:compression2-free}
In this proof, we assume that $K>N$ and $T=\max\big\{N\times\lceil{N}/{C}\rceil,\lceil{K}/{2}\rceil\big\}$ since implementing the identity function is sufficient otherwise.
The proof of Lemma \ref{lem:compression2-free} is similar to that of Lemma \ref{lem:compression1-free}.
To begin with, we first define a network $f_b(x):[0,K)\rightarrow[0,T)$ for 
$b=(b_i)_{i=1}^{C}\in[T]^{C}$ and $T=:M_1<M_2<\dots<M_{C+1}:=K$ so that $|\{x\in\mathcal X:M_i\le x<M_{i+1}\}|\le\lceil\frac{N}{C}\rceil\le\lfloor\tfrac{T}{N}\rfloor$ for all $i$ as
\begin{align}
    &f_b(x):=\begin{cases}x&\text{if}~x<T\\
    (x+b_1)\bmod T&\text{if}~M_1=T\le x<M_2\\
    (x+b_2)\bmod T&\text{if}~M_2\le x<M_3\\
    &\vdots\\
    (x+b_{C})\bmod T&\text{if}~M_{C}\le x
    \end{cases}\notag\\
    &=x-2T\times\mathbf{1}[x\ge T]\notag\\
    &\qquad+\sum_{i=1}^{C}\Bigg(\bigg(T+b_i-\sum_{j=1}^{i-1}b_j\bigg)\times\mathbf{1}[x\ge  M_i]-T\times\mathbf{1}\big[x\ge \min\{2T-b_i,M_{i+1}\}\big]\Bigg)\label{eq:compression2-free}
\end{align}
where \eqref{eq:compression2-free} holds as $T\ge\lceil\frac{K}2\rceil$.
Here, one can easily implement $f_b$ by a $\step+\id$ network of $L$ hidden layer and $d_\ell$ neurons at the $\ell$-th hidden layer by utilizing one neuron for storing the input $x$, another one neuron for storing the temporary output, and other neurons implement indicator functions at each layer. This is because one do not require to store $x$ in the last hidden layer and there exists $2C$ indicator functions to implement ($T\times\mathbf{1}[x\ge T]$ in \eqref{eq:compression2-free} can be absorbed into $\big(T+b_i-\sum_{j=1}^{i-1}b_j\big)\times\mathbf{1}[x\ge M_i]$ in \eqref{eq:compression2-free} for $i=1$).

Now, we show that if $T\ge N$, then there exist $b\in[T]^{C}$ such that $\big|\lfloor f_{b}(\mathcal X)\rfloor \big|=N$ to completes the proof. 
Our proof utilizes the mathematical induction on $i$: If there exist $b_1,\dots,b_{i-1}\in[T]$ such that 
\begin{align}
\lfloor f_b(\{x\in\mathcal X:x< M_i\})\rfloor=\binom{[T]}{\big|\lfloor\{x\in\mathcal X:x< M_i\}\rfloor\big|},\label{eq:compression2-free-ih}
\end{align}
then there exists $b_i\in[T]$ such that \begin{align}
\lfloor f_b(\{x\in\mathcal X:x< M_{i+1}\})\rfloor=\binom{[T]}{\big|\lfloor\{x\in\mathcal X:x< M_{i+1}\}\rfloor\big|}.\label{eq:compression2-free-ir}
\end{align}
Here, one can observe that the statement trivially holds for the base case, i.e., for $\{x\in\mathcal X:x<T\}$.
From the induction hypothesis, suppose that there exist $b_1,\dots,b_{i-1}\in[T]$ satisfying \eqref{eq:compression2-free-ih}. Now, we prove that there exists $b_i\in[T]$ such that $\mathcal S_{b_i}:=\lfloor f_b(\{x\in\mathcal X:M_i\le x<M_{i+1}\})\rfloor$ does not intersect with $\mathcal T:=\lfloor f_b(\{x\in\mathcal X:x\le M_i\})\rfloor$, i.e., \eqref{eq:compression2-free-ir} holds.
Consider the following inequality:
\begin{align*}
    \sum_{b_i\in [T]}|\mathcal S_{b_i}\cap\mathcal T|
    &\le \lfloor\tfrac{T}{N}\rfloor\times \big(N-\lfloor\tfrac{T}{N}\rfloor\big)<T
\end{align*}
where the first inequality follows from the fact that $|\mathcal S_{b_i}|\le \lfloor\tfrac{T}{N}\rfloor$, $|\mathcal T|\le \big(N-\lfloor\tfrac{T}{N}\rfloor\big)$, and for each $x\in\{x\in\mathcal X:M_i\le x<M_{i+1}\}$, there exists exactly $|\mathcal T|$ values of $b_i$ so that $\lfloor x\bmod T\rfloor\in\mathcal T$.
However, since the number of possible choices of $b_i$ is $T$,  there exists $b_i\in[T]$ such that $\mathcal S_{b_i}\cap\mathcal T=\emptyset$, i.e, \eqref{eq:compression2-free-ir} holds.
This completes the proof of Lemma \ref{lem:compression2-free}.

\newpage
\subsection{Proof of Lemma \ref{lem:learning}}\label{sec:pflem:learning}
In this proof, we explicitly construct $f_\theta$ satisfying the desired property stated in Lemma \ref{lem:learning}.
To begin with, we describe the high-level idea of the construction.
First, we construct a map $g:x\mapsto\big(\big\lfloor\frac{x}{B}\big\rfloor,x\bmod B\big)$ to transform an input $x\in[0,K)$ to a pair $(a,b)$ such that $a\in[A]$ and $\lfloor b\rfloor\in[B]$. Here, we give labels to the pair $(a,b)$ corresponding to the input $x$ as $y(a,\lfloor b\rfloor):=y(\lfloor x\rfloor)$. Note that this label is well-defined as if $\lfloor x_1\rfloor\ne\lfloor x_2\rfloor$, then $\lfloor g(x_1)\rfloor\ne \lfloor g(x_2)\rfloor$.
Now, we construct parameters $w_0,\dots,w_{A-1}$ containing the label information of $\mathcal X$ as
\begin{align}
    w_{a}:=\sum_{c\in[B]}y(a,c)\times 2^{-(c+1)D},
\end{align}
i.e., from the $(\lfloor b\rfloor\cdot D+1)$-th bit to the $(\lfloor b\rfloor\cdot D+D)$-th bit of the $a$-th parameter ($w_a$) contains the label information of $(a,\lfloor b\rfloor)$. 
Under this construction, we recover the label of $x\in\mathcal X$ by first mapping $x$ to the pair $(a,b)$ and extracting the parameter $w_a.$
Then, we extract bits from the $(\lfloor b\rfloor\cdot D+1)$-th bit to the $(\lfloor b\rfloor\cdot D+D)$-th bit of $w_a$ to recover the label $y(x)$.

Now, we explicitly construct the network mapping $x\in\mathcal X$ to $y(x)$. To this end, we introduce the following lemma. The proof of Lemma \ref{lem:paramextract-free} is presented in Appendix \ref{sec:pflem:paramextract-free}.
\begin{lemma}\label{lem:paramextract-free}
Suppose that $A, B, K, L, d_1, \dots d_L\in\mathbb N$ satisfy $AB\ge K$, $d_\ell\ge2$ for all $\ell$, and $\sum_{\ell=1}^L(d_\ell-2)\ge A$. Then, for any finite set $\mathcal X\subset[0,K)$ and $w_0,\dots,w_{A-1}\in\mathbb{R}$, there exists a $\step+\id$ network $f$ of $L$ layers and $d_\ell$ neurons at the $\ell$-th layer such that
$f(x)=(w_{\lfloor x/B\rfloor}, x\bmod B)$ for all $x\in\mathcal X$.
\end{lemma}
From Lemma \ref{lem:paramextract-free}, a $\step+\id$ network of $1$ hidden layer consisting of $A+2$ hidden neurons can map $x\in\mathcal X$ to $(w_{\lfloor x/B\rfloor}, x\bmod B)$. Note that this network requires overall $4A+10$ parameters ($3A+6$ edges and $A+4$ biases).

Finally, we introduce the following lemma for extracting from the $(\lfloor b\rfloor\cdot D+1)$-th bit to the $(\lfloor b\rfloor\cdot D+D)$-th bit of $w_{a}(=w_{\lfloor x/B\rfloor})$. The proof of Lemma \ref{lem:bitextract} is presented in Appendix~\ref{sec:pflem:bitextract}.
\begin{lemma}\label{lem:bitextract}
For any $D,B,R\in\mathbb N$, for any finite set $\mathcal X\subset[0,B)$, there exists a $\step+\id$ network $f$ of $2\lceil\frac{BD}{R}\rceil$ hidden layers and $\big((2R+5)2^R+2R^2+8R+7\big)\lceil\frac{BD}{R}\rceil-R2^R-R^2+3$ parameters satisfying the following property: For any $w=\sum_{i=1}^{BD} u_i\times2^{-i}$ such that $u_i\in\{0,1\}$, $f(x,w)=\sum_{i=1}^Du_{\lfloor x\rfloor\cdot D+i}\times2^{D-i}$ for all $x\in\mathcal X$.
\end{lemma}

From Lemma \ref{lem:bitextract}, a $\step+\id$ network of $2\lceil\frac{BD}{R}\rceil$ hidden layers and $\big((2R+5)2^R+2R^2+8R+7\big)\lceil\frac{BD}{R}\rceil-R2^R-R^2+3$ parameters.
Hence, by combining Lemma \ref{lem:paramextract-free} and Lemma \ref{lem:bitextract}, $f_\theta$ can be implemented by a $\step+\id$ network of $2\lceil\frac{BD}{R}\rceil+2$ hidden layers and $4A+\big((2R+5)2^R+2R^2+8R+7\big)\lceil\frac{BD}{R}\rceil-R2^R-R^2+3$ parameters. This completes the proof of Lemma \ref{lem:learning}.

\newpage
\subsection{Proof of Lemma \ref{lem:paramextract-free}}\label{sec:pflem:paramextract-free}
We design $f$ as $f:=f_L\circ\cdots\circ f_1(0,x)$ where each $f_\ell$ represents the function of the $\ell$-th layer consisting of $d_\ell$ neurons.
In particular, we construct $f_\ell$ as follows:
\begin{align*}
    f_\ell(w,x):=\bigg(&w+w_0\times\mathbf{1}[\ell=1]+\sum_{i=1}^{d_\ell-2}(w_{c_\ell+i}-w_{c_\ell+i-1})\times\mathbf{1}[x\ge iB],\\
    &x-\sum_{i=1}^{d_\ell-2}B\times\mathbf{1}[x\ge iB]\bigg)
\end{align*}
where $c_\ell:=\sum_{i=1}^{\ell-1}(d_\ell-2)$. 
Then, $f$ is the desired function and each $f_\ell$ can be implemented by a $\step+\id$ networks of $1$ hidden layer consisting of $d_\ell$ hidden neurons (two neurons for storing $x,w$ and other neurons are for $d_\ell-2$ indicator functions).
This completes the proof of Lemma \ref{lem:paramextract-free}.

\newpage
\subsection{Proof of Lemma \ref{lem:bitextract}}\label{sec:pflem:bitextract}
We construct $f(x,w):=2^{R\lceil BD/R\rceil+D}\times f_{\lceil BD/R\rceil}\circ\cdots\circ f_1(x,w)$ where $f_\ell$ is defined as
\begin{align*}
    f_\ell(x,v):=\begin{cases}
    \Big(x,v-\sum_{i=1}^Ru_{(\ell-1)R+i}\times2^{-(\ell-1)R-i}\\
    \qquad\quad+\sum_{i=1}^{R}\big(u_{(\ell-1)R+i}\wedge \mathbf{1}[m_{i,\ell}\le x<m_{i,\ell}+1]\big)\times2^{-r_{i,\ell}}\Big)~&\text{if}~\ell<\lceil\tfrac{BD}R\rceil\\
    v-\sum_{i=1}^Ru_{(\ell-1)R+i}\times2^{-(\ell-1)R-i}\\
    \quad\qquad+\sum_{i=1}^{R}\big(u_{(\ell-1)R+i}\wedge \mathbf{1}[m_{i,\ell}\le x<m_{i,\ell}+1]\big)\times2^{-r_{i,\ell}}~&\text{if}~\ell=\lceil\tfrac{BD}R\rceil
    \end{cases}.
\end{align*}
where $u_{i}$ denotes the $i$-th  bit of $w$ in the binary representation, $\wedge$ denotes the binary `and' operation, and $m_{i,\ell},r_{i,\ell}$ are defined as  
\begin{align*}
m_{i,\ell}&:=\left\lfloor\frac{(\ell-1)R+i-1}{D}\right\rfloor\\
r_{i,\ell}&:=\left\lceil\frac{BD}R\right\rceil R+\big((\ell-1) R+i-1\bmod D\big)+1.
\end{align*}
Namely, each $f_\ell$ extracts $R$ bits from the input $w$ and it store the extracted bits to the last bits of $v$ if the extracted bits are in from $(\lfloor x\rfloor\cdot D+1)$-th bit to the $(\lfloor x\rfloor\cdot D+D)$-th bit of $w$. 
Thus, $f(x,w)$ is the desired function for Lemma \ref{lem:bitextract}.

To implement each $f_\ell$ by a $\step+\id$ network, we introduce Lemma \ref{lem:subbitextract}. Note that we extract $u_i$ from $w$ in Lemma \ref{lem:subbitextract}, i.e., we do not assume that $u_i$ is given.
From Lemma \ref{lem:subbitextract}, a $\step+\id$ network of $2\lceil\frac{BD}{R}\rceil$ hidden layers consisting of $2^R+R+1$ and $R+2$ hidden neurons alternatively can map $(x,w)$ to $\sum_{i=1}^Du_{\lfloor x\rfloor\cdot D+i}\times2^{D-i}$ for all $x\in\mathcal X$.
By considering the input dimension $2$ and the output dimension $1$, this network requires $\big((2R+5)2^R+2R^2+8R+7\big)\lceil\frac{BD}{R}\rceil-R2^R-R^2+3$ parameters ($\big((2R+4)2^R+2R^2+6R+4\big)\lceil\frac{BD}{R}\rceil-R2^R-R^2+2$ edges and $(2^R+2R+3)\lceil\frac{BD}{R}\rceil+1$ biases).
This completes the proof of Lemma \ref{lem:bitextract}.

\begin{lemma}\label{lem:subbitextract}
A $\step+\id$ network of $2$ hidden layers having $2^R+R+1$ and $R+2$ hidden neurons at the first and the second hidden layer, respectively, can implement $f_\ell$.
\end{lemma}
\begin{proof}
We construct $f_\ell:=g_3\circ (g_2\oplus g_1)$ where $g_2\oplus g_1(x,v):=(g_2(x,v),g_1(x,v))$, i.e., the function concatenating the outputs of $g_1,g_2$.
In this proof, we mainly focus on constructing $f_\ell$ for $\ell<\lceil\frac{BD}{R}\rceil$ since $f_{\lceil{BD}/{R}\rceil}$ can be implemented similarly.
We define
$g_1,g_2,g_3$ as
\begin{align*}
    g_1(x,v):=&\Bigg(x,v,\sum_{i=0}^{2^{R}-1}\eta_{1,i}\times\mathbf{1}[ i\times2^{-\ell R}\le v<(i+1)\times2^{-\ell R}],\\
    &\qquad\qquad\qquad\vdots,\\
    &\qquad~~\sum_{i=0}^{2^{R}-1}\eta_{R,i}\times\mathbf{1}[ i\times2^{-\ell R}\le v<(i+1)\times2^{-\ell R}]\Bigg)\\
    =&(x,v,u_{(\ell-1)R+1},\dots,u_{\ell R})\\
    g_2\big(x,v):=&\Big(\mathbf{1}[m_{1,\ell}\le x<m_{1,\ell}+1],\dots,\mathbf{1}[m_{R,\ell}\le x<m_{R,\ell}+1]\Big)\\
\end{align*}
\begin{align*}
    g_3\circ(g_1\oplus g_2):=&\Bigg(x,v-\sum_{i=1}^Ru_{(\ell-1)R+i}\times2^{-(\ell-1)R-i})\\
    &\qquad\quad+\sum_{i=1}^{R}(u_{(\ell-1)R+i}\wedge \mathbf{1}[m_{i,\ell}\le x<m_{i,\ell}+1])\times2^{-r_{i,\ell}}\Bigg)\\
    =&\Bigg(x,v-\sum_{i=1}^Ru_{(\ell-1)R+i}\times2^{-(\ell-1)R-i}\\
    &\qquad\quad+\sum_{i=1}^{R}\mathbf{1}\big[u_{(\ell-1)R+i}+\mathbf{1}[m_{i,\ell}\le x<m_{i,\ell}+1]\ge2\big]\times2^{-r_{i,\ell}}\Bigg).
\end{align*}
where $\eta_{r,i}$ is a constant such that $\eta_{r,i}=1$ if $i\times 2^{-\ell R}\le x<(i+1)\times2^{-\ell R}$ implies that the $((\ell-1)R+r)$-th bit of $x$ is $1$ and $\eta_{r,i}=0$ otherwise.
Here, one can easily observe that $g_1$ can be implemented by a 
linear combinations of $\mathbf{1}[v\ge2^{-\ell R}],\dots,\mathbf{1}[v\ge(2^R-1)\times2^{-\ell R}]$ as it trivially holds that $\mathbf{1}[v\ge0]$ and $\mathbf{1}[v<2^{-(\ell-1)R}]$, i.e., $2^R-1$ indicator functions are enough for $g_1$.
Hence, $g_1$ can be implemented by a $\step+\id$ network
of $1$ hidden layer consisting of $2^R+1$ hidden neurons where additional $2$ neurons are for passing $x,v$.
In addition, $g_2$ can be implemented by a $\step+\id$ network of $R$ hidden neurons.
Finally, $g_3$ can be implemented by a $\step+\id$ network of $1$ hidden layer consisting of $R+2$ hidden neurons ($R$ neurons for $R$ indicator functions and $2$ neurons for passing $x,v$).

Therefore, $f_\ell$ can be implemented by a $\step+\id$ network of $2$ hidden layers consisting of $2^R+R+1$ hidden neurons for the first hidden layer and $R+2$ hidden neurons for the second hidden layer.
Note that implementation within two hidden layer is possible since the outputs of $g_1,g_2$ are simply linear combination of their hidden activation values and hence, can be absorbed into the linear map between hidden layers.
This completes the proof of Lemma \ref{lem:subbitextract}.
\end{proof}

\newpage
\subsection{Proof of Lemma \ref{lem:learning-free2}}\label{sec:pflem:learning-free2}
The main idea of the proof of Lemma \ref{lem:learning-free2} is identical to that of Lemma \ref{lem:learning}.
Recall $A, B$ and $w_0,\dots,w_{A-1}\in\mathbb R$ from the proof of Lemma \ref{lem:learning}.
Then, Lemma \ref{lem:learning-free2} is a direct corollary of Lemma \ref{lem:paramextract-free} and Lemma \ref{lem:bitextract-3}.
\begin{lemma}\label{lem:bitextract-3}
For any $D,B\in\mathbb N$, for any finite set $\mathcal X\subset[0,B)$, for any $w=\sum_{i=1}^{DB} u_i\times2^{-i}$ for some $u_i\in\{0,1\}$, there exists a $\step+\id$ network $f$ of $(2D+1)B$ hidden layers and width 3 such that $f(x,w)=\sum_{i=1}^Du_{\lfloor x\rfloor\cdot D+i}\times2^{-i}$ for all $x\in\mathcal X$.
\end{lemma}
\begin{proof}
We construct $f(x,w):=2^{D(B+1)}\times f_{DB}\circ g_{DB}\circ h_{DB}\circ\cdots\circ f_1\circ g_1\circ h_1$ where $f_\ell,g_\ell,h_\ell$ are defined as
\begin{align*}
    &h_\ell(x,w):=\begin{cases}\big(x,w\big)\quad&\text{if}~\ell\bmod D\ne1\\
    \big(x-1+3K\times\mathbf{1}[x<0],w\big)\quad&\text{if}~\ell\bmod D=1
    \end{cases}\\
    &g_\ell(x,w):=\big(x,w-2^{-\ell}\times\mathbf{1}[w\ge2^{-\ell}],\mathbf{1}[w\ge2^{-\ell}]\big)\\
    &f_\ell(x,w,n)=\big(x, w+2^{-r_\ell}\times\mathbf{1}[x-n<-1]\big)
\end{align*}
where $r_\ell:=DK+(\ell-1\mod D)+1$.
Now, we explain the constructions of $f_\ell,g_\ell,h_\ell$.
Let $x\in[0,B)$ and $w\in[0,1)$ be inputs of $f$, i.e., consider $f(x,w)$.
First, the indicator function in $h_\ell$ is activated only at $\ell=(\lfloor x\rfloor+1)\cdot D+1$ as $x<B$. In particular, the first entry of the output of $h_{(\lfloor x\rfloor+1)\cdot D+1}$ is greater than $2B$ and this is the maximum value of the first entry of the output of $h_{(\lfloor x\rfloor+1)\cdot D+1}$ as it monotonically decreases as $\ell$ grows.
The indicator function in $g_\ell$ extracts and outputs the $\ell$-th bit of $w$.
Lastly, $f_\ell$ add the $\ell$-th bit of $w$ extracted by $g_\ell$ to the $\big(DK+(\ell-1\mod D)+1\big)$-th bit of $w$ if and only if $\ell\in\{\lfloor x\rfloor\cdot D+1,\dots,(\lfloor x\rfloor+1)\cdot D\}$. This is because $x\in[-1, 0)$ if and only if $\ell\in\{\lfloor x\rfloor\cdot D+1,\dots,(\lfloor x\rfloor+1)\cdot D\}$.

Here, $h_\ell$ at $\ell\bmod D=1$ can be implemented by a $\step+\id$ network of $1$ hidden layer and width $3$, $g_\ell$ can be implemented by a $\step+\id$ network of $1$ hidden layer and width $3$, and $f_\ell$ can be implemented by a $\step+\id$ network of $1$ hidden layer and width 3.
Hence, $f$ can be implemented by a $\step+\id$ network of $(2D+1)B$ hidden layers and width 3. This completes the proof of Lemma \ref{lem:bitextract-3}.
\end{proof}

\end{document}